\documentclass{article}
\usepackage{iclr2027_conference,times}
\usepackage[utf8]{inputenc}
\usepackage[T1]{fontenc}
\usepackage{url}
\usepackage{microtype}
\usepackage{graphicx}
\usepackage{fontawesome5}
\graphicspath{{figures/}}
\usepackage{booktabs}
\usepackage{bbm}
\usepackage{threeparttable}
\usepackage{amsmath,amssymb,mathtools,amsthm}
\usepackage{xcolor}
\usepackage{enumitem}
\usepackage{multirow}
\usepackage{array}
\usepackage{etoolbox}
\usepackage{arxiv-layout}
\usepackage{hyperref}
\usepackage[capitalize,noabbrev]{cleveref}

\theoremstyle{plain}
\newtheorem{theorem}{Theorem}[section]
\newtheorem{proposition}[theorem]{Proposition}
\newtheorem{lemma}[theorem]{Lemma}
\newtheorem{corollary}[theorem]{Corollary}
\theoremstyle{definition}
\newtheorem{definition}[theorem]{Definition}
\newtheorem{assumption}[theorem]{Assumption}
\theoremstyle{remark}
\newtheorem{remark}[theorem]{Remark}

\newcommand{\R}{\mathbb{R}}
\newcommand{\1}{\mathbbm{1}}
\newcommand{\E}{\mathbb{E}}
\newcommand{\Rad}{\mathfrak{R}}
\newcommand{\Var}{\operatorname{Var}}
\newcommand{\Cov}{\operatorname{Cov}}
\newcommand{\tr}{\operatorname{tr}}
\newcommand{\Risk}{\mathcal{R}}
\newcommand{\MSCE}{\operatorname{MSCE}}
\newcommand{\HET}{\operatorname{HET}}
\newcommand{\Vproc}{\mathcal V_{\mathrm{proc}}}
\DeclareMathOperator*{\argmin}{arg\,min}

\makeatletter
\patchcmd{\@maketitle}
  {\begin{tabular}[t]{l}\bf\rule{\z@}{24pt}\@author\end{tabular}}
  {\centerline{\begin{tabular}[t]{c}\bf\rule{\z@}{24pt}\@author\end{tabular}}}
  {}{\PackageWarning{main}{Could not center the ICLR author block}}
\makeatother

\title{\Large How Synthetic Labels Improve Conformal Prediction\\[-0.1ex]
A Perspective on Conditional Coverage}
\author{Qianyi Chen \qquad Bo Li\\
School of Economics and Management, Tsinghua University}

\iclrfinalcopy
\hypersetup{
  pdftitle={How Synthetic Labels Improve Conformal Prediction: A Perspective on Conditional Coverage},
  pdfauthor={Qianyi Chen and Bo Li},
  colorlinks=true,
  linkcolor=blue!55!black,
  citecolor=green!40!black,
  urlcolor=blue!65!black
}
\begin{document}
\maketitle
\lhead{}

\begin{abstract}
Conformal prediction provides distribution-free finite-sample marginal coverage, but post-hoc calibration data may be too scarce to learn how uncertainty varies across inputs. Meanwhile, abundant covariates can often be labeled cheaply by domain models or general-purpose language models. We study whether these synthetic labels can improve conditional coverage when only a small trusted sample is available. Building on score-quantile regression, we introduce prediction-powered quantile learning: a synthetic-labeled pool estimates pinball risk, paired trusted and synthetic outcomes correct its bias, and an independent trusted split performs final conformalization. Profiling pinball risk over scalar corrections reveals that population conditional-coverage error is its functional gradient; the corresponding Hessian removes global shifts and weights remaining shape error by boundary density. Composing this geometry with prediction-powered learning yields a three-resource expansion and a benefit--cost rule for synthetic power. Across eight regression benchmarks, synthetic-powered quantile learning substantially improves downstream conditional coverage while preserving marginal validity and producing more compact prediction sets. A human-rating study finds similar gains from external LLM labels and exposes a quality--quantity--cost tradeoff.
Code and data are available at \mbox{\small\href{https://github.com/Cqyiiii/synthetic-label-conformal}{\raisebox{-0.1ex}{\faGithub}}\hspace{0.3em}\url{https://github.com/Cqyiiii/synthetic-label-conformal}}.
\end{abstract}

\section{Introduction}
High-stakes machine learning requires reliable uncertainty estimates, not only accurate point predictions. Conformal prediction (CP) is attractive because it is distribution free, model agnostic, and gives finite-sample guarantees that a prediction set contains the future outcome at a user-specified rate~\cite{vovk2005algorithmic,shafer2008tutorial}.

The canonical split-conformal guarantee is nevertheless marginal: it averages coverage over the test population. A procedure may therefore under-cover difficult inputs while compensating with overly wide sets on easy ones. This issue is especially important under heteroscedasticity, when predictive uncertainty varies substantially across the feature space. Although exact distribution-free feature-conditional coverage is impossible in general with finite samples~\cite{vovk2012conditional,lei2014distribution,foygel2021limits}, improving the conditional coverage behavior of conformal procedures remains a central practical objective.

Rectified Conformal Prediction (RCP)~\cite{plassier2025rectifying} addresses this problem by learning a covariate-dependent quantile of a fixed nonconformity score, then conformalizing its residual with a scalar correction. The quantile adapts the threshold to local difficulty, while the correction restores finite-sample marginal coverage. Because pinball-loss excess risk controls score-quantile error, conditional-coverage improvement becomes a quantile-learning problem rather than a full conditional-distribution problem~\cite{plassier2025rectifying,chen2025colorful}.

The main obstacle is data scarcity. CP is usually applied post hoc to an already trained predictor, so the original training data may be unavailable and the only new trusted outcomes may form a modest calibration reservoir. RCP must divide this reservoir between learning a heterogeneous score threshold and applying the final conformal correction. With only tens or hundreds of trusted scores, the threshold can be unstable even when marginal coverage is reliable.

Classical semi-supervised learning offers additional covariates but no additional outcomes. In general, the marginal distribution $P_X$ does not identify the conditional law $P_{Y\mid X}$: two populations may share exactly the same $P_X$ while having arbitrarily different response mechanisms. Unlabeled covariates can therefore help only under additional structural assumptions, such as smoothness, low-density separation or cluster structure, and manifold assumptions~\cite{chapelle2006ssl}. Rather than infer $P_{Y\mid X}$ from $P_X$ alone, we study a richer prediction-powered regime. We observe a small sample with trusted outcomes and a much larger sample of covariates from the same target population; an independently trained auxiliary labeler---for example, an existing domain model or a general-purpose LLM---provides synthetic outcomes at these covariates. Such labels may be biased, but they carry outcome-like information that unlabeled covariates alone do not provide.

This regime is increasingly common. Computational models can assign surrogate properties to millions of candidate molecules before expensive laboratory assays, and automated inspection systems can score sensor records before destructive verification. Recent work uses unlabeled covariates or generated scores to refine conformal sets when trusted data are scarce~\cite{seedat2023improving,zhou2025semi,bashari2025syntheticpowered}. Our question is different: once an independent trusted split guarantees marginal validity, can synthetic labels improve conditional coverage?

We address this question by augmenting the quantile-learning component of conformal prediction. Inspired by prediction-powered inference~\cite{angelopoulos2023prediction,angelopoulos2023ppi++}, we construct a bias-corrected empirical pinball risk: a large synthetic-labeled sample estimates an auxiliary risk, while paired trusted and synthetic outcomes correct its bias. We apply the idea to conditional outcome and nonconformity-score quantiles, and compare it with supervised, semi-supervised, transport, pretraining, and positive-augmentation baselines. Every method reserves an untouched trusted split for final calibration.

The final correction is not merely a validity device; it changes the value of information learned upstream. We formalize this through the profiled pinball risk obtained after optimizing over the scalar correction. Its functional gradient is the population-calibrated coverage error, and its Hessian is the density-weighted operator that removes global threshold shifts. Composing this geometry with the learner's influence function identifies the true--synthetic covariance that can affect deployment. The resulting expansion separates trusted quantile learning, synthetic-pool estimation, and final calibration, and yields an explicit power rule balancing covariance benefit against synthetic noise.

Our experiments complete this learning-to-deployment story. Across eight classical regression benchmarks and 30 common splits, synthetic-powered quantile learning substantially improves downstream conditional-coverage diagnostics and consistently produces more compact prediction sets. Intermediate pinball and post-conformal analyses connect better quantile estimation to final coverage, while a matched permutation control separates input-specific information from generic stabilization. A realistic HelpSteer2 study further shows that ratings from general-purpose LLMs can improve multivariate conditional coverage under a practical quality--quantity--cost tradeoff.

\paragraph{Contributions.}
\begin{itemize}[leftmargin=*,itemsep=2pt,topsep=3pt]
    \item We formulate a setting with abundant covariates, few trusted labels, and an auxiliary labeler, and propose prediction-powered quantile learning for conditional coverage.
    \item Profiling over the final scalar correction identifies calibrated coverage error as a functional gradient and the density-weighted post-conformal map as its Hessian. Combining this geometry with prediction-powered influence yields a three-resource expansion and a benefit--cost power rule.
    \item Across eight regression datasets and HelpSteer2, we link synthetic-powered quantile learning to deployed coverage and separate genuine pairing from generic stabilization.
\end{itemize}

\subsection{Related Works}\label{sec:related_works}
\paragraph{Conditional coverage and adaptive conformal prediction.}
Split CP guarantees marginal coverage, while exact distribution-free feature-conditional validity is generally unattainable in finite samples~\cite{vovk2012conditional,lei2014distribution,foygel2021limits}. Existing work therefore studies group-conditional guarantees, localized coverage, training-conditional guarantees, and held-out tests of conditional miscoverage~\cite{jung2023batch,ding2023class,kiyani2024conformal,gibbs2025conformal,bian2023training,hore2024local,cauchois2021knowing,braun2025conditionalcoveragediagnosticsconformal}. Another line improves conditional behavior by changing the score or estimating a richer conditional distribution~\cite{romano2019conformalized,lei2018distribution,xie2024boosted,izbicki2022cd,plassier2025probabilistic,colombo2024normalizing}. Most closely related, RCP estimates a conditional score quantile and conformalizes its residual~\cite{plassier2025rectifying}. Colorful Pinball emphasizes boundary density~\cite{chen2025colorful}; we characterize how that density shapes post-conformal error.

\paragraph{Semi-supervised and synthetic-powered conformal prediction.}
Additional covariates alone do not identify $P_{Y\mid X}$: the same marginal $P_X$ is compatible with arbitrarily different conditional response laws. Classical semi-supervised learning therefore relies on extra structure, such as smoothness, low-density separation or cluster assumptions, and manifold geometry~\cite{chapelle2006ssl}. Semi-supervised conformal methods use unlabeled covariates for interval refinement, reweighting, or score matching~\cite{seedat2023improving,zhou2025semi}. Synthetic-Powered Predictive Inference (SPPI) instead combines trusted and generated scores through a rank-based transporter and targets efficient calibration when trusted calibration data are scarce~\cite{bashari2025syntheticpowered}. We evaluate these approaches, but ask a different question: how can outcome-like synthetic labels improve conditional coverage when an independent trusted split is reserved for final conformalization?

\paragraph{Prediction-powered learning and inference.}
Prediction-powered inference combines a small trusted sample with abundant model predictions through a bias-corrected control variate~\cite{angelopoulos2023prediction,angelopoulos2023ppi++,ji2025predictions}. Related work studies prediction-powered optimization, semi-supervised quantile estimation, calibrated control variates, and multiple prediction sources~\cite{shoham2025predictionpowered}. We use the same correction to learn conditional quantiles, but the distinctive challenge is downstream: only learning error that survives the final correction can improve conditional coverage.

\section{Preliminaries and Problem Setting}\label{sec:preliminaries}
\paragraph{Split conformal prediction.}
Let $(X,Y)\sim P$ be a future test pair. We condition on a predictive model trained independently of all samples introduced below. A fixed nonconformity score $S(x,y)\in\R$ measures how incompatible a candidate outcome $y$ is with input $x$; for regression, a standard choice is $S(x,y)=|y-\widehat\mu(x)|$ or its multivariate $\ell_\infty$ analogue. Given an independent calibration sample $\{(X_r^c,Y_r^c)\}_{r=1}^m$, write $S_r^c=S(X_r^c,Y_r^c)$, set $\tau=1-\alpha$, and let $k=\lceil(m+1)\tau\rceil$. Split CP uses the $k$th order statistic $S_{(k)}^c$ and returns
\begin{equation}\label{eq:split_cp_set}
\mathcal C_{\rm split}(x)=\{y:S(x,y)\le S_{(k)}^c\}.
\end{equation}
Exchangeability of the calibration and test scores gives
\begin{equation}\label{eq:split_validity}
\Pr\{Y\in\mathcal C_{\rm split}(X)\}\ge \frac{k}{m+1}\ge\tau,
\end{equation}
With continuous scores, equality holds at $k/(m+1)$. We set the threshold to $+\infty$ when $k=m+1$.

\paragraph{Conditional coverage and RCP.}
For a fitted prediction set $\mathcal C_{\mathcal D}$, define its conditional coverage curve by $\pi_{\mathcal D}(x)=\Pr\{Y\in\mathcal C_{\mathcal D}(X)\mid X=x,\mathcal D\}$. Marginal coverage controls only $\E_X\pi_{\mathcal D}(X)$. We measure conditional miscalibration through
\begin{equation}\label{eq:msce_definition}
\MSCE_{\mathcal D}=\E_X[(\pi_{\mathcal D}(X)-\tau)^2],
\qquad
\HET_{\mathcal D}=\Var_X\{\pi_{\mathcal D}(X)\}.
\end{equation}
RCP~\cite{plassier2025rectifying} replaces the global threshold in \eqref{eq:split_cp_set} by a learned function $q(x)$. On a separate trusted calibration sample, it forms residuals $R_r^c=S_r^c-q(X_r^c)$ and lets $\widehat\gamma(q)$ be their $k$th order statistic. The resulting set is
\begin{equation}\label{eq:final_set}
\mathcal C_q(x)=\{y:S(x,y)\le q(x)+\widehat\gamma(q)\}.
\end{equation}
Here $q$ provides local adaptation and $\widehat\gamma(q)$ restores the usual marginal guarantee.

\paragraph{Learning the conditional score quantile.}
Let $F_x(s)=\Pr(S\le s\mid X=x)$ and define $q^\star(x)=F_x^{-1}(\tau)$. The oracle set $\{y:S(x,y)\le q^\star(x)\}$ has conditional coverage $\tau$ whenever $F_x$ is continuous at $q^\star(x)$. We learn $q^\star$ through the pinball loss $\rho_\tau(u)=u\{\tau-\1(u<0)\}$, with $\ell_q(x,s)=\rho_\tau\{s-q(x)\}$ and population risk $\Risk(q)=\E\ell_q(X,S)$.

Our analysis begins with the risk-to-coverage bridge established for RCP~\cite{plassier2025rectifying}.
\begin{proposition}[Pinball excess risk controls pre-conformal coverage error]\label{prop:risk_coverage_bridge}
Assume $F_x(q^\star(x))=\tau$ and that $F_x$ is $L_F$-Lipschitz uniformly in $x$. Then every measurable $q$ satisfies
\begin{equation}\label{eq:risk_coverage_bridge}
\E\big[(F_X\{q(X)\}-\tau)^2\big]
\le 2L_F\{\Risk(q)-\Risk(q^\star)\}.
\end{equation}
\end{proposition}
A proof is given in Appendix~\ref{app:proof_risk_coverage}. The proposition controls coverage error before the final scalar correction; Section~\ref{sec:post_conformal_geometry} derives how that correction changes the relevant error.

\paragraph{Synthetic-supervision setting.}\label{sec:synthetic_setting}
The base predictor, auxiliary labeler, and preprocessing maps are fitted independently of the samples below and are then treated as fixed. We observe a small trusted sample $\mathcal D_L=\{(X_i,Y_i)\}_{i=1}^n$ and a much larger covariate-only sample $\mathcal D_U^X=\{\widetilde X_j\}_{j=1}^N$, where $N\gg n$. The covariates in both samples are independent draws from the same target distribution $P_X$, but true outcomes are observed only in $\mathcal D_L$.

The auxiliary labeler supplies $Y_i^\dagger$ at each trusted covariate and $\widetilde Y_j$ at each unlabeled covariate. We write $S_i=S(X_i,Y_i)$ for the trusted score, $S_i^\dagger=S(X_i,Y_i^\dagger)$ for the paired synthetic score, and $\widetilde S_j=S(\widetilde X_j,\widetilde Y_j)$ for the synthetic-pool score. The same labeler is applied in both places, so $(X,S^\dagger)\overset d=(\widetilde X,\widetilde S)$, although this common synthetic law need not equal the trusted law of $(X,S)$. Finally, an independent trusted sample $\mathcal D_C=\{(X_r^c,Y_r^c)\}_{r=1}^m$ is reserved exclusively for the scalar correction in \eqref{eq:final_set}. Thus trusted learning data identify the target risk, the large synthetic pool estimates an auxiliary risk, and untouched trusted calibration data restore finite-sample coverage.

\section{Methodology and Theory}\label{sec:3}
\paragraph{Roadmap.}
We proceed in three steps. First, we define prediction-powered quantile learning and show that an untouched calibration split preserves exact marginal validity. Second, we profile pinball risk over the final scalar correction, revealing which threshold errors remain visible after conformalization. Third, we pass prediction-powered learning fluctuations through this geometry to obtain the three-resource expansion and its synthetic-power rule. Appendix~\ref{app:additional_results} states the regularity and secondary results, while Appendix~\ref{app:parametric_proofs} supplies the learner construction.

\subsection{Prediction-powered quantile learning}
Let $P_n$ denote the empirical average over $\mathcal D_L$ and $\widetilde P_N$ the empirical average over the synthetic pool. We fit the conditional score quantile with
\begin{equation}\label{eq:ppi_obj_main}
\widehat\Risk_\lambda(q)
=P_n\ell_q(X,S)
+\lambda\{\widetilde P_N\ell_q(\widetilde X,\widetilde S)-P_n\ell_q(X,S^\dagger)\}.
\end{equation}
The choice $\lambda=0$ recovers supervised learning, while $\lambda=1$ gives unit PPI. Because paired and pool synthetic scores share a law, $\E\widehat\Risk_\lambda(q)=\Risk(q)$ for every $q$ and $\lambda$, even under a biased labeler. The subtraction therefore removes the population bias of the pool term.

For comparison, positive augmentation omits the paired subtraction and may bias the population threshold, although untouched calibration still preserves marginal validity for the final prediction set.

Independent final calibration separates marginal validity from the quality of the learned threshold. We first formalize this protection, then decompose the conditional error that remains.
\begin{proposition}[Exact validity under adaptive synthetic learning]\label{prop:adaptive_validity}
Let $\widehat q$ and any power or tuning rule be measurable functions of $(\mathcal D_L,\mathcal D_U^X)$ and randomness independent of $\mathcal D_C$ and the test pair. If \eqref{eq:final_set} uses the exact $k$th calibration order statistic, then
\begin{equation}\label{eq:adaptive_validity}
\Pr\{Y\in\mathcal C_{\widehat q}(X)\mid\mathcal D_L,\mathcal D_U^X\}
\ge k/(m+1)\ge\tau.
\end{equation}
Under continuous residuals, equality holds at $k/(m+1)$.
\end{proposition}
Thus synthetic learning and adaptive power selection may change conditional efficiency without invalidating marginal coverage, provided the final calibration sample remains untouched. Validity alone, however, does not describe a fitted procedure's conditional behavior.

\begin{proposition}[Realized conditional-error decomposition]\label{prop:msce_exact}
For a fitted procedure, the mean coverage $U_{\mathcal D}=\E_X\pi_{\mathcal D}(X)$ satisfies
\begin{equation}\label{eq:msce_exact_decomp}
\MSCE_{\mathcal D}=\HET_{\mathcal D}+(U_{\mathcal D}-\tau)^2.
\end{equation}
If the marginal residual CDF is continuous, then conditional on $\widehat q$,
\begin{equation}\label{eq:beta_setup}
U_{\mathcal D}\stackrel d=U_{(k)}\sim\operatorname{Beta}(k,m+1-k),
\qquad
\E[(U_{(k)}-\tau)^2]=:\sigma_{m,\tau}^2.
\end{equation}
Averaging over fitted procedures also gives
\begin{equation}\label{eq:realized_ensemble_decomp}
\E_{\mathcal D}\MSCE_{\mathcal D}
=\E_X[(\E_{\mathcal D}\pi_{\mathcal D}(X)-\tau)^2]
+\E_X\Var_{\mathcal D}\{\pi_{\mathcal D}(X)\}.
\end{equation}
\end{proposition}
Together, the identities separate covariate heterogeneity from finite-calibration offset and require evaluation within each fitted run before aggregation across repetitions.

\subsection{What survives final conformalization?}\label{sec:post_conformal_geometry}
Final conformalization treats threshold level and shape differently: it cancels every constant shift exactly at finite sample, so only covariate-dependent shape can alter the deployed set.
\begin{proposition}[Global shift invariance]\label{prop:shift_invariance}
For every measurable $q$ and constant $c$,
\begin{equation}\label{eq:shift_invariance}
\widehat\gamma(q+c)=\widehat\gamma(q)-c,
\qquad
\mathcal C_{q+c}(x)=\mathcal C_q(x).
\end{equation}
\end{proposition}
Thus deployment depends on the shape of $q$, not its global level. We now make this quotient structure explicit at the population level. Define the profiled correction and risk
\begin{equation}\label{eq:profiled_risk_definition}
\gamma_q\in\argmin_{\gamma\in\R}\Risk(q+\gamma),
\qquad
\overline\Risk(q)=\Risk(q+\gamma_q),
\qquad
\pi_q(x)=F_x\{q(x)+\gamma_q\}.
\end{equation}
When the residual CDF is continuous, the first-order condition gives $\E_X\pi_q(X)=\tau$.

\begin{proposition}[Profiled risk and post-conformal coverage]\label{prop:profiled_risk}
Assume that $\gamma_q$ is locally unique and differentiation may pass through expectation. For every constant $c$ and admissible direction $h$,
\begin{align}
\overline\Risk(q+c)&=\overline\Risk(q),\qquad
D\overline\Risk(q)[h]=\E_X[\{\pi_q(X)-\tau\}h(X)],\label{eq:profiled_gradient}\\
\|\nabla\overline\Risk(q)\|_{L_2(P_X)}^2
&=\E_X[\{\pi_q(X)-\tau\}^2].\label{eq:profiled_msce_identity}
\end{align}
If $f_q(x)=f_{S\mid X}\{q(x)+\gamma_q\mid x\}$ exists and $\bar f_q=\E f_q(X)>0$, then
\begin{equation}\label{eq:profiled_hessian}
(\mathcal H_qh)(x)=f_q(x)\left[h(x)-\frac{\E\{f_q(X)h(X)\}}{\bar f_q}\right],
\qquad \mathcal H_q=\nabla^2\overline\Risk(q).
\end{equation}
\end{proposition}
Thus population-calibrated conditional error is the functional gradient of a shift-invariant risk, and the density-weighted post-conformal operator is its curvature. Appendix~\ref{app:profiled_proof_details} records the derivative of $\gamma_q$, positive semidefiniteness, and the quotient-space quadratic form. Applying Proposition~\ref{prop:risk_coverage_bridge} to $q+\gamma_q$ gives
\begin{equation}\label{eq:profiled_risk_bridge}
\|\nabla\overline\Risk(q)\|_2^2
\le 2L_F\{\overline\Risk(q)-\Risk(q^\star)\}.
\end{equation}
The inequality controls magnitudes, not rankings: profiled risk and conditional-coverage error can order the same candidates differently.

For a target level $u$, let $\Gamma(q,u)$ solve $\E_XF_X\{q(X)+\Gamma(q,u)\}=u$ and set $\mathcal P(q,u)(x)=F_x\{q(x)+\Gamma(q,u)\}$. At $q^\star$, the profiled Hessian gives the local deployment map.
\begin{corollary}[Local density-weighted projection]\label{prop:functional_projection}
Set $f_\star(x)=f_{S\mid X}\{q^\star(x)\mid x\}$ and $\bar f_\star=\E f_\star(X)>0$. Under the smoothness conditions of Appendix~\ref{app:profiled_proof_details},
\begin{align}
\mathcal P(q^\star+e,\tau+v)-\tau
&=\mathcal Le+\frac{f_\star}{\bar f_\star}v+r_{e,v},
&\|r_{e,v}\|_2&\le C(\|e\|_\infty+|v|)^2,\label{eq:functional_coverage_expansion}\\
(\mathcal Le)(x)
&=f_\star(x)\left[e(x)-\frac{\E\{f_\star(X)e(X)\}}{\bar f_\star}\right].\label{eq:functional_operator}
\end{align}
\end{corollary}
The density converts threshold error into coverage error, while the scalar correction removes its density-weighted global component. Thus $\mathcal L=\mathcal H_{q^\star}$ acts on threshold shape modulo constant shifts.

To connect this function-space derivative to a learned model, let $q_0=q_{\theta_0}$ be an interior stationary point of the trusted pinball risk, and let $\gamma_0$ satisfy $\E_XF_X\{q_0(X)+\gamma_0\}=\tau$. Write $\pi_0(x)=F_x\{q_0(x)+\gamma_0\}$, $d(x)=\nabla_\theta q_\theta(x)|_{\theta_0}$, $J=\nabla_\theta^2\Risk(q_\theta)|_{\theta_0}$, and
\begin{equation}\label{eq:projected_geometry}
f_c(x)=f_{S\mid X}\{q_0(x)+\gamma_0\mid x\},
\qquad
d_c(x)=d(x)-\frac{\E\{f_c(X)d(X)\}}{\E f_c(X)}.
\end{equation}
When the tangent contains an unpenalized intercept, stationarity already implies $\gamma_0=0$; we retain it to cover restricted or penalized learners.

\begin{definition}[Post-conformal influence map]\label{def:post_conformal_map}
For a learning-score perturbation $v$, define
\begin{equation}\label{eq:post_conformal_map}
(\mathsf Bv)(x)=f_c(x)d_c(x)^\top J^{-1}v,
\qquad
G_c=J^{-1}\E[f_c(X)^2d_c(X)d_c(X)^\top]J^{-1}.
\end{equation}
Then $\|\mathsf Bv\|_{L_2(P_X)}^2=v^\top G_cv$.
\end{definition}
The inverse Hessian maps a score perturbation to a threshold change, while $f_cd_c$ retains only the component visible after conformalization; their composition $\mathsf Bv$ is the deployed coverage perturbation.

\subsection{When do synthetic labels help?}\label{sec:when_help}
At the learning boundary define
\begin{equation}\label{eq:influence_scores}
\xi=\{\1(S\le q_0(X))-\tau\}d(X),
\qquad
\xi^\dagger=\{\1(S^\dagger\le q_0(X))-\tau\}d(X),
\end{equation}
with pool analogue $\widetilde\xi$. The deployment-scale constants are
\begin{equation}\label{eq:abc_defs}
a=\tr\{G_c\Var(\xi)\},\qquad
b=\tr\{G_c\Cov(\xi,\xi^\dagger)\},\qquad
c=\tr\{G_c\Var(\xi^\dagger)\},
\end{equation}
and $C_{\rm cal}=\E[(f_c(X)/\E f_c(X))^2]$. Thus $a$, $b$, and $c$ respectively quantify trusted deployment variance, usable true--synthetic covariance, and the synthetic variance cost.

Appendix~\ref{app:additional_results}, Assumption~\ref{ass:deployment_qmd}, states the fixed-dimensional conditions: stationary curvature, a uniform asymptotic-linear learner, a smooth $L_2(P_X)$ deployment map, independent samples, and proportional rates. Appendix~\ref{app:parametric_proofs} constructs a one-step estimator.

Define
\begin{equation}\label{eq:power_quadratic}
L(\lambda)=\frac{a-2b\lambda+(1+\kappa)c\lambda^2}{n},
\qquad
\sigma_{m,\tau}^2=\E[(U_{(k)}-\tau)^2].
\end{equation}
\begin{theorem}[Three-resource post-conformal expansion]\label{thm:realized_msce_expansion}
Under Assumption~\ref{ass:deployment_qmd}, the following expansions hold uniformly over the fixed admissible power set:
\begin{align}
\E_{\mathcal D}\|\pi_{\mathcal D,\lambda}-\pi_0\|_2^2
&=L(\lambda)+C_{\rm cal}\sigma_{m,\tau}^2+o(n^{-1}+N^{-1}+m^{-1}),\label{eq:stationary_stoch_expansion}\\
\Vproc(\lambda)
&=L(\lambda)+C_{\rm cal}\Var(U_{(k)})+o(n^{-1}+N^{-1}+m^{-1}).\label{eq:stationary_variance_expansion}
\end{align}
\end{theorem}
The expansion separates three resources. Trusted learning contributes $a/n$; paired covariance supplies the benefit $-2b\lambda/n$; paired and pool estimation incur $(1+\kappa)c\lambda^2/n$. Final calibration contributes a separate order-$1/m$ term that no synthetic pool can remove.

The theorem controls stochastic error around the calibrated stationary target $\pi_0$. Under exact score-quantile specification this target equals the desired constant coverage curve, so the same expansion becomes a statement about total MSCE.
\begin{corollary}[Exact score-quantile specification]\label{cor:exact_specification}
If $q_0=q^\star$, then $\pi_0\equiv\tau$ and
\begin{align}
\E\MSCE_{\mathcal D}
&=L(\lambda)+C_{\rm cal}\sigma_{m,\tau}^2+o(n^{-1}+N^{-1}+m^{-1}),\label{eq:exact_msce}\\
\E\HET_{\mathcal D}
&=L(\lambda)+(C_{\rm cal}-1)\sigma_{m,\tau}^2+o(n^{-1}+N^{-1}+m^{-1}),\label{eq:exact_het}
\end{align}
where $C_{\rm cal}-1=\Var\{f_\star(X)\}/\bar f_\star^2$.
\end{corollary}
If the boundary density is constant across inputs, finite calibration produces a common coverage offset but no first-order heterogeneity. Density variation converts the same probability-scale error into different coverage changes across the covariate space.

The covariance benefit $b$ has two statistically distinct sources.
\begin{proposition}[Two sources of useful synthetic information]\label{prop:mechanism_decomposition}
For $\mu(X)=\E(\xi\mid X)$ and $\mu^\dagger(X)=\E(\xi^\dagger\mid X)$,
\begin{equation}\label{eq:two_mechanisms}
b=
\underbrace{\tr\!\left[G_c\E\{\Cov(\xi,\xi^\dagger\mid X)\}\right]}_{b_{\rm res}}
+
\underbrace{\tr\!\left[G_c\Cov\{\mu(X),\mu^\dagger(X)\}\right]}_{b_{\rm align}}.
\end{equation}
Exact specification removes $b_{\rm align}$, whereas $S^\dagger\perp S\mid X$ removes $b_{\rm res}$.
\end{proposition}
Residual coupling is within-$X$ information about whether trusted and synthetic scores fall on the same side of the relevant boundary. Alignment is between-$X$ information about where a restricted learner is systematically too high or too low. A frozen deterministic teacher is $X$-only conditional on its training data, so real-data gains are compatible primarily with the alignment channel; residual coupling is isolated in simulation.

\begin{corollary}[No first-order PPI gain under exact independent synthetic scores]\label{cor:no_gain}
Under exact specification, if $S^\dagger\perp S\mid X$, then $b=0$ and $\lambda=0$ is a first-order optimum over nonnegative powers, uniquely so when $c>0$.
\end{corollary}
This is a boundary for the matched bias-robust correction in Eq.~\eqref{eq:ppi_obj_main}, not a claim that independently generated synthetic data are intrinsically useless. If their full joint law is externally certified to equal the target law, direct pooling can increase effective sample size; Appendix~\ref{app:additional_results}, Proposition~\ref{prop:pooling_robustness}, quantifies the corresponding robustness--efficiency tradeoff.

\begin{corollary}[Benefit--cost rule for synthetic power]\label{cor:optimal_power}
For $c>0$ and admissible interval $\Lambda$,
\begin{equation}\label{eq:optimal_lambda}
\lambda^\star=\Pi_\Lambda\left\{\frac{b}{(1+\kappa)c}\right\}.
\end{equation}
When clipping is inactive,
\begin{equation}\label{eq:lambda_regret}
L(\lambda)-L(\lambda^\star)=\frac{(1+\kappa)c}{n}(\lambda-\lambda^\star)^2.
\end{equation}
Unit power improves on supervision if and only if $2b>(1+\kappa)c$.
\end{corollary}
The rule gives a direct interpretation: $b$ is the usable covariance benefit, $c$ is the synthetic-variance price, and $\kappa=n/N$ records pool-estimation precision. Positive projected covariance supports some positive power but does not justify unit power. Appendix~\ref{app:additional_results} gives the certified-pooling robustness comparison (Proposition~\ref{prop:pooling_robustness}), an estimable-power rule (Corollary~\ref{cor:plugin_power}), and the trusted-label allocation result (Corollary~\ref{cor:label_allocation}).

\section{Experimental Study}\label{sec:experiments}
The empirical study asks three connected questions. First, does synthetic-powered quantile learning improve conditional coverage relative to the supervised method in the same conformal family? Second, are the gains attributable to input-specific synthetic information rather than only additional optimization or regularization? Third, how do labeler quality, the synthetic-pool size, trusted-label scarcity, and final calibration affect the result? We answer these questions with a common audited protocol, then examine a realistic external-LLM setting.

\subsection{Benchmark design}\label{sec:experimental_setup}
\paragraph{Datasets.}
We use eight public regression benchmarks spanning scalar, multivariate, and high-dimensional outcomes. Table~\ref{tab:dataset_summary} summarizes the processed data used by every method. Bike Sharing and Diamonds are large scalar-regression problems; Gas Turbine, Naval Propulsion, SGEMM, and Transcoding have two or four outputs; WEC has 49 outputs and tests whether one score radius can control a high-dimensional prediction box. Dataset sources, preprocessing, target definitions, and license notes appear in Appendix~\ref{app:datasets}.
\begin{table}[htbp]
\centering
\caption{Regression benchmarks used in the main study. $d_X$ and $d_Y$ are the processed covariate and response dimensions. SGEMM is evaluated on a fixed 50,000-row subsample; all other datasets use every cleaned observation.}
\label{tab:dataset_summary}
\begin{tabular*}{\linewidth}{@{\extracolsep{\fill}}lrrrl@{}}
\toprule
Dataset & $N$ & $d_X$ & $d_Y$ & Application domain \\
\midrule
Bike Sharing & 17,379 & 13 & 1 & Urban mobility \\
Diamonds & 53,940 & 23 & 1 & Price prediction \\
Gas Turbine & 36,733 & 9 & 2 & Power-plant emissions \\
Naval Propulsion & 11,934 & 16 & 2 & Predictive maintenance \\
SGEMM & 50,000 & 14 & 4 & GPU-kernel performance \\
Superconductivity & 21,263 & 81 & 1 & Materials science \\
Transcoding & 68,784 & 23 & 2 & Video processing \\
WEC & 27,000 & 98 & 49 & Wave-energy farms \\
\bottomrule
\end{tabular*}
\end{table}

\paragraph{Data roles and label budget.}
Each seed begins with one random permutation and disjoint maximum reservoirs. One percent of the observations fit preprocessing, 16\% train the auxiliary labeler, 4\% validate it, up to 40\% provide covariates for synthetic labeling, up to 4\% train the base predictor or CQR endpoints, and up to 2\% form the conformal reservoir. Separate samples define conditional-coverage groups, select worst slices, and evaluate all methods. The main low-label configuration uses a 30\% synthetic pool, 2\% base training, and a 1\% conformal reservoir. RCP splits the latter equally between score-quantile learning and untouched final calibration; Split, SPPI, and NNM use the full reservoir for calibration. Resource sweeps take nested prefixes of fixed reservoirs, so changing one resource never changes the evaluation rows. Ordinary synthetic methods do not read hidden outcomes in the synthetic pool.

This accounting is deliberately strict. The auxiliary labeler is trained externally to the downstream conformal procedure, while the final calibration sample is never used to fit a representation, select a power, choose a checkpoint, or estimate a conditional metric. These separations make the finite-sample coverage guarantee and the three-resource interpretation operational rather than only conceptual.

\paragraph{Models and methods.}
The tabular auxiliary labeler is a random forest trained only on its designated trusted outcomes. All CQR and RCP variants use two-hidden-layer width-128 ReLU networks, common initialization within each dataset--seed family, and matched optimization budgets; they differ only in the learning objective and in which synthetic data that objective receives. All final conformal corrections use the exact order statistic.

Split is ordinary split conformal prediction with trusted residual scores. SPPI~\cite{bashari2025syntheticpowered} transports ranks between trusted and synthetic score distributions. NNM~\cite{zhou2025semi} corrects synthetic scores by matching them to nearby trusted observations. CQR~\cite{romano2019conformalized} learns lower and upper outcome quantiles and conformalizes their endpoint error; CQR-PPI replaces the supervised endpoint losses by the bias-corrected prediction-powered objective in Eq.~\eqref{eq:ppi_obj_main}.

RCP learns a conditional quantile of the $\ell_\infty$ residual score and then applies an independent scalar correction. RCP-PTFT pretrains the same quantile network on synthetic scores and fine-tunes it on trusted scores. RCP-PPI-CV fits prediction-powered candidates and selects $\lambda\in\{0,.25,.5,.75,1\}$ using trusted pinball validation. RCP-Aug minimizes $P_n\ell_q(X,S)+0.5\widetilde P_N\ell_q(\widetilde X,\widetilde S)$; its positive synthetic risk can improve finite-sample stability but is not bias robust. The proposed RCP-PPI uses Eq.~\eqref{eq:ppi_obj_main} at the preregistered unit power. Our most controlled comparisons---CQR versus CQR-PPI and RCP versus RCP-PPI---hold architecture, initialization, data roles, optimization, and final calibration fixed, changing only the quantile-learning objective. Appendix~\ref{app:method_details} gives all solver and architecture details.

\paragraph{Conditional-coverage metrics and uncertainty.}
We report three complementary diagnostics. Independent data fit $K=30$ covariate groups; for group masses $\widehat p_k$ and held-out group coverages $\widehat c_k$, grouped MSCE is
\begin{equation}\label{eq:group_raw_msce_main}
\widehat{\MSCE}_{30}=\sum_{k=1}^{30}\widehat p_k(\widehat c_k-\tau)^2.
\end{equation}
Worst-Slice Coverage (WSC) selects a low-coverage slab on a dedicated sample and evaluates it on different observations. $L_1$-ERT fits a cross-validated coverage-probability model and measures predictable deviation from the nominal target. Lower MSCE and $L_1$-ERT and higher WSC are better. We also report marginal coverage, original-scale log volume, empty/full-set frequencies, failures, and a split-half grouped estimator that removes the leading Bernoulli evaluation-noise term. All principal entries are mean $\pm$ standard deviation over 30 common splits. Paired bootstrap intervals compare a synthetic method with the supervised method in the same family on identical splits and evaluation rows; no cross-dataset average is treated as an inferential estimand. Formal definitions appear in Appendix~\ref{app:metric_details}.

\subsection{Main conditional-coverage benchmark}\label{sec:main_results}
Tables~\ref{tab:main_msce}--\ref{tab:main_ert} report the complete ten-method benchmark under the moderate synthetic labeler.
\begin{table}[htbp]
\centering
\caption{Grouped mean-squared conditional-coverage error (MSCE; lower is better) on Bike Sharing, Diamonds, Gas Turbine, Naval Propulsion, SGEMM, Superconductivity, Transcoding, and WEC. We report $10^3\!\times$ mean $\pm$ standard deviation over 30 common splits. Rules separate method families; bold marks the best value in each column, and abbreviated headers are defined in Section~\ref{sec:experimental_setup}.}
\label{tab:main_msce}
\begin{tabularx}{\linewidth}{@{}l*{4}{>{\centering\arraybackslash}X}@{}}
\toprule
Method & Bike & Diam. & Gas & Naval \\
\midrule
Split & 5.89$\pm$1.58 & 13.10$\pm$2.62 & 8.13$\pm$2.08 & 10.15$\pm$3.87 \\
SPPI & 6.92$\pm$2.18 & 14.01$\pm$2.60 & 10.05$\pm$2.54 & 11.99$\pm$4.83 \\
NNM & 6.13$\pm$2.09 & 12.99$\pm$2.85 & 8.16$\pm$2.00 & 11.02$\pm$4.86 \\
\midrule
CQR & 3.39$\pm$1.19 & 6.06$\pm$1.66 & 2.45$\pm$0.71 & 7.47$\pm$2.84 \\
CQR-PPI & 2.60$\pm$1.68 & 3.36$\pm$0.74 & \textbf{2.14$\pm$0.62} & \textbf{5.54$\pm$1.98} \\
\midrule
RCP & 5.76$\pm$3.16 & 7.12$\pm$2.31 & 4.42$\pm$2.19 & 10.29$\pm$5.74 \\
RCP-PTFT & 4.98$\pm$2.64 & 5.58$\pm$1.88 & 3.43$\pm$1.48 & 9.27$\pm$5.33 \\
RCP-PPI-CV & 2.90$\pm$1.56 & 4.35$\pm$2.43 & 3.64$\pm$1.57 & 7.64$\pm$4.66 \\
RCP-Aug & 3.40$\pm$1.62 & 3.67$\pm$0.88 & 3.21$\pm$1.38 & 7.60$\pm$3.86 \\
\textbf{RCP-PPI} & \textbf{2.52$\pm$1.55} & \textbf{3.33$\pm$1.22} & 3.01$\pm$1.17 & 6.41$\pm$2.57 \\
\bottomrule
\end{tabularx}

\tablepanelgap
\begin{tabularx}{\linewidth}{@{}l*{4}{>{\centering\arraybackslash}X}@{}}
\toprule
Method & SGEMM & Cond. & Trans. & WEC \\
\midrule
Split & 5.94$\pm$1.53 & 8.12$\pm$2.55 & 5.68$\pm$1.83 & 13.28$\pm$4.36 \\
SPPI & 6.16$\pm$1.23 & 11.05$\pm$4.32 & 6.38$\pm$1.71 & 15.98$\pm$4.79 \\
NNM & 5.86$\pm$1.04 & 9.02$\pm$3.67 & 5.64$\pm$1.26 & 13.13$\pm$4.35 \\
\midrule
CQR & 3.89$\pm$1.03 & 6.07$\pm$2.00 & 4.12$\pm$1.44 & 8.85$\pm$2.73 \\
CQR-PPI & \textbf{1.04$\pm$0.25} & 4.46$\pm$1.28 & \textbf{2.59$\pm$0.58} & 9.86$\pm$3.34 \\
\midrule
RCP & 3.68$\pm$1.61 & 7.02$\pm$2.55 & 4.75$\pm$2.36 & 6.71$\pm$3.44 \\
RCP-PTFT & 3.17$\pm$1.54 & 6.48$\pm$2.15 & 4.19$\pm$2.46 & \textbf{5.38$\pm$2.54} \\
RCP-PPI-CV & 2.78$\pm$2.20 & 4.29$\pm$1.86 & 3.78$\pm$2.98 & 7.88$\pm$3.58 \\
RCP-Aug & 2.26$\pm$1.09 & 4.38$\pm$1.04 & 3.50$\pm$2.46 & 5.92$\pm$2.98 \\
\textbf{RCP-PPI} & 3.09$\pm$2.13 & \textbf{4.07$\pm$1.98} & 3.20$\pm$2.81 & 8.00$\pm$3.39 \\
\bottomrule
\end{tabularx}
\end{table}

\begin{table}[htbp]
\centering
\caption{Worst-Slice Coverage (WSC) over the same 30 common splits; higher is better. Rules separate method families, and bold marks the best value in each dataset.}
\label{tab:main_wsc}
\begin{tabularx}{\linewidth}{@{}l*{4}{>{\centering\arraybackslash}X}@{}}
\toprule
Method & Bike & Diam. & Gas & Naval \\
\midrule
Split & .843$\pm$.048 & .620$\pm$.051 & .719$\pm$.057 & .752$\pm$.080 \\
SPPI & .817$\pm$.053 & .604$\pm$.048 & .686$\pm$.059 & .705$\pm$.080 \\
NNM & .838$\pm$.050 & .621$\pm$.052 & .723$\pm$.061 & .731$\pm$.088 \\
\midrule
CQR & .867$\pm$.047 & .763$\pm$.046 & .827$\pm$.035 & .817$\pm$.062 \\
CQR-PPI & .883$\pm$.044 & .810$\pm$.028 & \textbf{.854$\pm$.032} & \textbf{.828$\pm$.057} \\
\midrule
RCP & .855$\pm$.072 & .724$\pm$.056 & .803$\pm$.066 & .760$\pm$.098 \\
RCP-PTFT & .857$\pm$.073 & .761$\pm$.052 & .827$\pm$.046 & .779$\pm$.089 \\
RCP-PPI-CV & \textbf{.891$\pm$.053} & .798$\pm$.056 & .821$\pm$.059 & .817$\pm$.083 \\
RCP-Aug & .882$\pm$.059 & .803$\pm$.031 & .839$\pm$.048 & .807$\pm$.080 \\
\textbf{RCP-PPI} & .890$\pm$.051 & \textbf{.816$\pm$.041} & .839$\pm$.049 & .825$\pm$.068 \\
\bottomrule
\end{tabularx}

\tablepanelgap
\begin{tabularx}{\linewidth}{@{}l*{4}{>{\centering\arraybackslash}X}@{}}
\toprule
Method & SGEMM & Cond. & Trans. & WEC \\
\midrule
Split & .739$\pm$.049 & .790$\pm$.043 & .745$\pm$.050 & .781$\pm$.040 \\
SPPI & .737$\pm$.052 & .758$\pm$.055 & .732$\pm$.049 & .758$\pm$.045 \\
NNM & .741$\pm$.056 & .785$\pm$.051 & .746$\pm$.047 & .785$\pm$.047 \\
\midrule
CQR & .784$\pm$.036 & .818$\pm$.052 & .779$\pm$.049 & .821$\pm$.035 \\
CQR-PPI & \textbf{.875$\pm$.021} & .847$\pm$.043 & .824$\pm$.034 & .810$\pm$.042 \\
\midrule
RCP & .797$\pm$.038 & .802$\pm$.066 & .795$\pm$.039 & .808$\pm$.052 \\
RCP-PTFT & .817$\pm$.035 & .810$\pm$.057 & .809$\pm$.035 & \textbf{.822$\pm$.046} \\
RCP-PPI-CV & .829$\pm$.049 & .843$\pm$.046 & .813$\pm$.048 & .813$\pm$.053 \\
RCP-Aug & .835$\pm$.034 & .846$\pm$.051 & .827$\pm$.036 & .818$\pm$.053 \\
\textbf{RCP-PPI} & .837$\pm$.041 & \textbf{.851$\pm$.044} & \textbf{.832$\pm$.048} & .813$\pm$.048 \\
\bottomrule
\end{tabularx}
\end{table}

\begin{table}[htbp]
\centering
\caption{$L_1$ Excess Risk of Target coverage ($L_1$-ERT) over 30 common splits. Entries are mean $\pm$ standard deviation; lower is better.}
\label{tab:main_ert}
\begin{tabularx}{\linewidth}{@{}l*{4}{>{\centering\arraybackslash}X}@{}}
\toprule
Method & Bike & Diam. & Gas & Naval \\
\midrule
Split & .064$\pm$.008 & .127$\pm$.006 & .052$\pm$.005 & .062$\pm$.017 \\
SPPI & .066$\pm$.009 & .129$\pm$.004 & .058$\pm$.007 & .063$\pm$.017 \\
NNM & .062$\pm$.008 & .127$\pm$.005 & .053$\pm$.006 & .064$\pm$.016 \\
\midrule
CQR & .039$\pm$.009 & .085$\pm$.010 & .031$\pm$.008 & .047$\pm$.012 \\
CQR-PPI & .035$\pm$.014 & .062$\pm$.006 & \textbf{.028$\pm$.009} & \textbf{.042$\pm$.012} \\
\midrule
RCP & .058$\pm$.014 & .090$\pm$.008 & .044$\pm$.010 & .065$\pm$.021 \\
RCP-PTFT & .052$\pm$.012 & .075$\pm$.009 & .039$\pm$.009 & .059$\pm$.023 \\
RCP-PPI-CV & .038$\pm$.016 & .063$\pm$.019 & .039$\pm$.011 & .054$\pm$.020 \\
RCP-Aug & .041$\pm$.013 & .059$\pm$.008 & .037$\pm$.009 & .051$\pm$.017 \\
\textbf{RCP-PPI} & \textbf{.033$\pm$.017} & \textbf{.051$\pm$.009} & .036$\pm$.009 & .045$\pm$.015 \\
\bottomrule
\end{tabularx}

\tablepanelgap
\begin{tabularx}{\linewidth}{@{}l*{4}{>{\centering\arraybackslash}X}@{}}
\toprule
Method & SGEMM & Cond. & Trans. & WEC \\
\midrule
Split & .112$\pm$.005 & .095$\pm$.013 & .073$\pm$.004 & .112$\pm$.010 \\
SPPI & .111$\pm$.005 & .108$\pm$.017 & .075$\pm$.005 & .120$\pm$.010 \\
NNM & .112$\pm$.005 & .098$\pm$.017 & .073$\pm$.005 & .113$\pm$.010 \\
\midrule
CQR & .087$\pm$.009 & .070$\pm$.011 & .070$\pm$.007 & .090$\pm$.008 \\
CQR-PPI & .034$\pm$.007 & \textbf{.056$\pm$.009} & .059$\pm$.004 & .092$\pm$.008 \\
\midrule
RCP & .065$\pm$.010 & .081$\pm$.010 & .053$\pm$.009 & .091$\pm$.010 \\
RCP-PTFT & .051$\pm$.011 & .075$\pm$.008 & .049$\pm$.007 & .081$\pm$.008 \\
RCP-PPI-CV & .034$\pm$.015 & .063$\pm$.011 & .040$\pm$.011 & .077$\pm$.010 \\
RCP-Aug & .039$\pm$.011 & .066$\pm$.007 & .042$\pm$.008 & \textbf{.074$\pm$.009} \\
\textbf{RCP-PPI} & \textbf{.031$\pm$.010} & .060$\pm$.008 & \textbf{.034$\pm$.007} & .078$\pm$.008 \\
\bottomrule
\end{tabularx}
\end{table}

Prediction-powered quantile learning consistently strengthens the relevant supervised quantile learner. CQR-PPI lowers grouped MSCE and raises WSC on seven of eight datasets, with particularly large MSCE reductions on SGEMM ($3.89\!\to\!1.04$), Naval ($7.47\!\to\!5.54$), and Transcoding ($4.12\!\to\!2.59$). RCP-PPI lowers grouped MSCE on seven datasets, improves every WSC point estimate, and lowers $L_1$-ERT on all eight. On Diamonds, MSCE falls from $7.12$ to $3.33$, WSC rises from $.724$ to $.816$, and $L_1$-ERT falls from $.090$ to $.051$. Its held-out pinball risk is lower than supervised RCP on every dataset (Appendix~\ref{app:supp_experiments}), providing an upstream learning check before the post-conformal evaluation.

The exceptions are informative rather than hidden. On WEC, the high-dimensional response makes RCP-PTFT strongest in grouped MSCE, while RCP-PPI still improves WSC and $L_1$-ERT. CQR-PPI improves seven datasets but not WEC grouped MSCE. No single numerical diagnostic is treated as the definition of conditional coverage: agreement among MSCE, WSC, $L_1$-ERT, the split-half audit, and prediction-set size is more informative than any one table.

The insertion point and robustness correction also matter. SPPI and NNM enlarge the calibration resource without explicitly learning heterogeneous score quantiles; under this protocol they remain close to or worse than Split CP on conditional metrics. Synthetic pretraining and positive augmentation frequently improve RCP, but their population targets can shift with labeler bias. RCP-PPI-CV often helps but is less stable than the preregistered fixed power when trusted validation data are scarce. The results therefore support prediction-powered quantile learning rather than an undifferentiated claim that every use of generated labels improves conformal prediction.

The gains are not bought by indiscriminate widening. Across all method--dataset cells, marginal coverage lies in $[0.882,0.919]$, no full prediction set occurs, and both CQR-PPI and RCP-PPI reduce mean log volume on every dataset relative to their supervised counterparts. Appendix~\ref{app:completed_results} reports the paired effects, coverage, log-volume, empty-set, and split-half audits.

\subsection{Where do the gains come from?}\label{sec:interpreting_gains}
Positive augmentation can help by regularizing an unstable threshold even when its labels contain little input-specific information. We therefore compare genuine augmentation with a matched permutation placebo that preserves the synthetic-score marginal distribution, pool size, network, initialization, optimizer, and total synthetic weight, while destroying the covariate--score pairing. Figure~\ref{fig:pairing_ablation} shows that genuine labels improve both grouped-MSCE estimators and WSC on five prespecified datasets. The accompanying numerical table and intervals are in Appendix~\ref{app:ablation_results}. The conclusion is deliberately narrow: generic stabilization is an important part of positive augmentation, but genuine pairing provides additional conditional information in these settings.

\begin{figure}[htbp]
    \centering
    \includegraphics[width=0.82\linewidth]{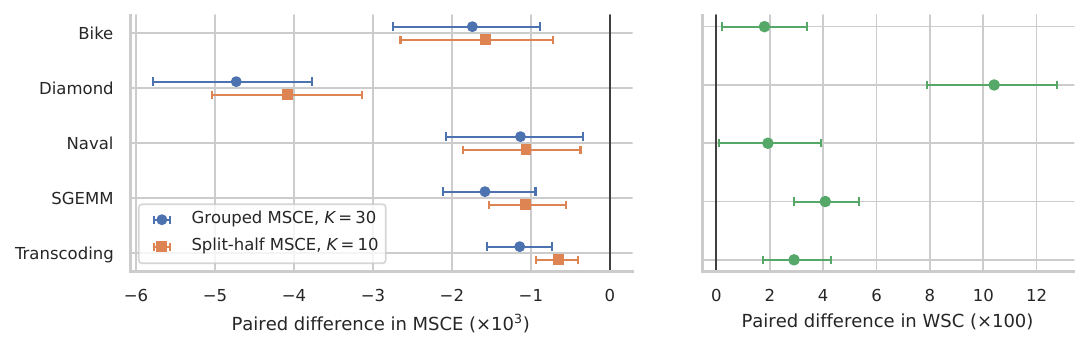}
    \caption{Genuine RCP augmentation versus a matched permutation placebo. Negative MSCE differences and positive WSC differences favor the genuine covariate--score pairing; bars are paired-bootstrap 95\% intervals over 20 common splits.}
    \label{fig:pairing_ablation}
\end{figure}

Teacher accuracy alone does not predict the value of synthetic labels. Table~\ref{tab:teacher_sensitivity} compares weak, moderate, and oracle labelers. Oracle labels generally define an upper ceiling, but the path is not uniformly monotone because post-conformal value depends on boundary-aligned information, the learner's tangent space, and synthetic variance rather than only global predictive accuracy. The trusted-label budget experiment in Table~\ref{tab:budget_sensitivity} shows a clearer pattern: gains are largest in the ultra-low and low regimes and usually shrink as more trusted labels stabilize the supervised learner.

\begin{table}[htbp]
\centering
\caption{Effect of synthetic-label quality. Entries are grouped MSCE $\times10^3$ for weak/moderate/oracle labelers on the same 10 splits; lower is better. Oracle labels use hidden outcomes and are an upper-bound diagnostic.}
\label{tab:teacher_sensitivity}
\setlength{\tabcolsep}{4.5pt}
\begin{tabular*}{\linewidth}{@{\extracolsep{\fill}}lccc@{}}
\toprule
Dataset & CQR-PPI & RCP-Aug & \textbf{RCP-PPI} \\
\midrule
Bike & 2.82/2.33/2.02 & 4.75/4.28/4.21 & 3.96/3.96/2.99 \\
Diamond & 4.48/3.52/0.65 & 3.57/4.18/3.20 & 3.85/4.02/1.13 \\
Gas Turbine & 2.29/2.04/1.05 & 2.74/2.74/2.48 & 3.35/3.00/1.47 \\
Naval & 7.65/5.47/2.92 & 8.38/6.91/6.10 & 9.20/6.06/4.05 \\
SGEMM & 2.59/1.04/0.77 & 2.51/2.38/2.31 & 3.95/3.18/2.90 \\
Supercond. & 5.64/4.53/2.76 & 4.52/4.34/3.40 & 5.57/3.45/1.77 \\
Transcoding & 3.19/2.52/1.08 & 4.16/3.89/3.85 & 5.58/3.75/1.51 \\
WEC & 11.75/10.18/5.97 & 10.38/7.24/3.04 & 9.26/9.35/1.69 \\
\bottomrule
\end{tabular*}
\end{table}

\begin{table}[htbp]
\centering
\caption{Trusted-label budget sensitivity. Entries are paired grouped-MSCE changes $\times10^3$ relative to the supervised method in the same family; negative is better. Ultra-low, low, and moderate use base/conformal proportions of 1\%/0.5\%, 2\%/1\%, and 4\%/2\%, respectively, on the same 10-seed support.}
\label{tab:budget_sensitivity}
\setlength{\tabcolsep}{5pt}
\begin{tabular*}{\linewidth}{@{\extracolsep{\fill}}llrrr@{}}
\toprule
Dataset & Method & Ultra-low & Low & Moderate \\
\midrule
Bike & CQR-PPI & -2.79 & -0.93 & -1.30 \\
Bike & RCP-Aug & -2.94 & -3.34 & -1.62 \\
Bike & \textbf{RCP-PPI} & -3.80 & -3.66 & -2.49 \\
\midrule
Diamond & CQR-PPI & -4.24 & -2.88 & -1.33 \\
Diamond & RCP-Aug & -5.55 & -4.14 & -2.01 \\
Diamond & \textbf{RCP-PPI} & -6.65 & -4.30 & -1.04 \\
\midrule
SGEMM & CQR-PPI & -3.77 & -2.82 & -2.01 \\
SGEMM & RCP-Aug & -2.41 & -1.36 & -0.56 \\
SGEMM & \textbf{RCP-PPI} & -1.49 & -0.56 & -0.46 \\
\midrule
Transcoding & CQR-PPI & -1.68 & -1.35 & -1.29 \\
Transcoding & RCP-Aug & -1.19 & -0.88 & -0.48 \\
Transcoding & \textbf{RCP-PPI} & -2.26 & -1.02 & -0.47 \\
\bottomrule
\end{tabular*}
\end{table}

The pool-size experiment in Table~\ref{tab:varying_N} changes only the synthetic prefix. A small pool already produces most of the observed gains, while further increases often plateau or become mildly nonmonotone. This is consistent with the three-resource theorem: increasing $N$ only reduces the pool-estimation component; it cannot manufacture missing covariance, repair model misspecification, or remove the independent final-calibration term.
\begin{table}[htbp]
\centering
\caption{Effect of the synthetic-pool size on RCP-PPI. Only the synthetic-pool prefix varies; trusted learning, final calibration, grouping, and evaluation roles are fixed. The 0\% row is supervised RCP. Entries are mean $\pm$ standard deviation over 10 common splits. MSCE is multiplied by $10^3$ (lower is better), WSC is higher-is-better, and $L_1$-ERT is lower-is-better.}
\label{tab:varying_N}
\setlength{\tabcolsep}{4.5pt}
\begin{tabular*}{\linewidth}{@{\extracolsep{\fill}}llccc@{}}
\toprule
Dataset & Synthetic share & Grouped MSCE & WSC & $L_1$-ERT \\
\midrule
Bike & 0\% & 7.62$\pm$3.37 & 0.818$\pm$0.083 & 0.063$\pm$0.016 \\
Bike & 5\% & 4.55$\pm$1.98 & 0.863$\pm$0.077 & 0.045$\pm$0.015 \\
Bike & 10\% & 3.66$\pm$2.05 & 0.873$\pm$0.067 & 0.042$\pm$0.019 \\
Bike & 20\% & 3.36$\pm$1.91 & 0.873$\pm$0.063 & 0.038$\pm$0.021 \\
Bike & 40\% & 3.40$\pm$2.18 & 0.869$\pm$0.056 & 0.038$\pm$0.022 \\
\midrule
Diamond & 0\% & 8.32$\pm$3.29 & 0.713$\pm$0.068 & 0.092$\pm$0.010 \\
Diamond & 5\% & 4.32$\pm$1.81 & 0.807$\pm$0.062 & 0.058$\pm$0.008 \\
Diamond & 10\% & 4.06$\pm$1.11 & 0.805$\pm$0.056 & 0.057$\pm$0.008 \\
Diamond & 20\% & 4.01$\pm$1.57 & 0.803$\pm$0.058 & 0.054$\pm$0.008 \\
Diamond & 40\% & 3.85$\pm$1.19 & 0.801$\pm$0.049 & 0.053$\pm$0.008 \\
\midrule
SGEMM & 0\% & 3.74$\pm$1.46 & 0.781$\pm$0.042 & 0.064$\pm$0.012 \\
SGEMM & 5\% & 2.11$\pm$1.08 & 0.867$\pm$0.044 & 0.030$\pm$0.008 \\
SGEMM & 10\% & 2.73$\pm$1.79 & 0.854$\pm$0.040 & 0.033$\pm$0.008 \\
SGEMM & 20\% & 3.41$\pm$2.95 & 0.840$\pm$0.050 & 0.034$\pm$0.010 \\
SGEMM & 40\% & 3.57$\pm$2.82 & 0.845$\pm$0.053 & 0.037$\pm$0.010 \\
\midrule
Transcoding & 0\% & 4.77$\pm$2.72 & 0.807$\pm$0.048 & 0.053$\pm$0.008 \\
Transcoding & 5\% & 3.89$\pm$3.19 & 0.833$\pm$0.072 & 0.038$\pm$0.007 \\
Transcoding & 10\% & 3.91$\pm$3.37 & 0.831$\pm$0.071 & 0.036$\pm$0.007 \\
Transcoding & 20\% & 3.91$\pm$3.70 & 0.826$\pm$0.059 & 0.036$\pm$0.008 \\
Transcoding & 40\% & 3.89$\pm$3.68 & 0.835$\pm$0.059 & 0.037$\pm$0.007 \\
\bottomrule
\end{tabular*}
\end{table}

\subsection{Direct checks of the post-conformal theory}\label{sec:theory_checks_main}
The benchmark establishes practical improvements, while controlled experiments test the mathematical mechanisms. Figure~\ref{fig:post_conformal_derivative} checks the profiled deployment derivative directly. Constant threshold shifts are numerical nulls, and the first-order remainder for linear and quadratic perturbations decays at slopes 1.981 and 2.009, matching the predicted second-order behavior. Figure~\ref{fig:three_resource_scaling} separately examines the resource laws. When $N=4n$ and $m=n$, deployment error decays at an empirical slope close to $-1$; at fixed $n$ and large $m$, the population-calibrated PPI error has the predicted linear dependence on $1/N$, while the supervised control is flat. Together with the calibration-size experiment in Appendix~\ref{app:theory_diagnostics}, these results test the $n^{-1}$, $N^{-1}$, and $m^{-1}$ components without treating simulation as the paper's empirical headline.

\begin{figure}[htbp]
    \centering
    \begin{minipage}[t]{0.485\linewidth}
        \centering
        \includegraphics[width=\linewidth]{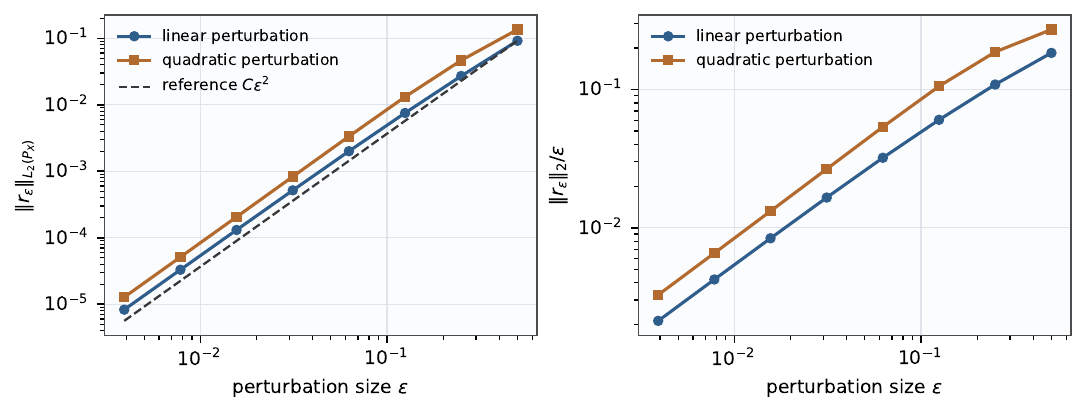}\\[-0.3ex]
        {\small (a) Population derivative and second-order remainder.}
    \end{minipage}\hfill
    \begin{minipage}[t]{0.485\linewidth}
        \centering
        \includegraphics[width=\linewidth]{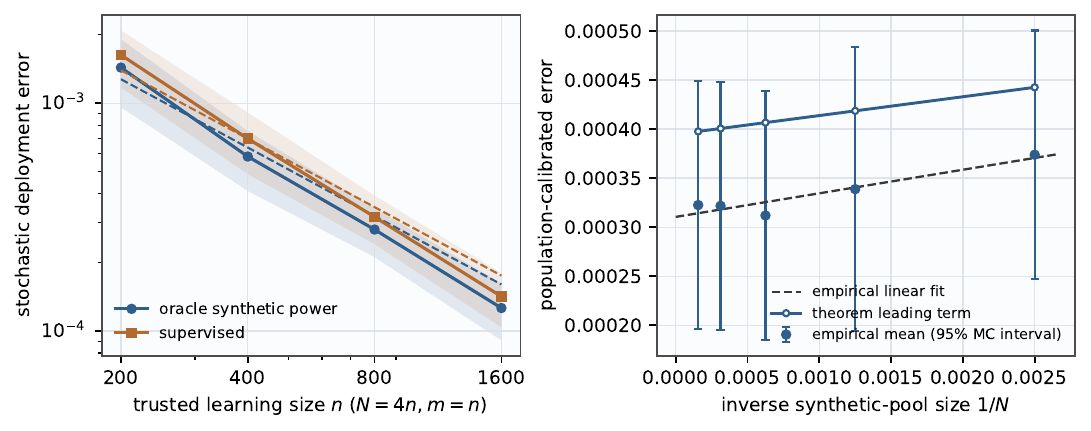}\\[-0.3ex]
        {\small (b) Joint and isolated resource scaling.}
    \end{minipage}
    \caption{Theory-aligned diagnostics. Panel (a) checks exact global-shift nullity and the local density-weighted derivative. Panel (b) checks the three-resource scaling predicted by Theorem~\ref{thm:realized_msce_expansion}. Full settings and numerical intervals are reported in Appendix~\ref{app:theory_followup}.}
    \label{fig:theory_checks}
    \label{fig:post_conformal_derivative}
    \label{fig:three_resource_scaling}
\end{figure}

\section{Realistic External-LLM Supervision}\label{sec:helpsteer_main}
The tabular benchmark uses domain-specific random forests as external labelers. We next ask whether the same idea remains useful when synthetic outcomes come from general-purpose language models and the target is a multivariate human judgment. HelpSteer2 contains five 0--4 ratings---helpfulness, correctness, coherence, complexity, and verbosity---for prompt--response pairs. After deduplication and prompt-level role separation, the study contains 21,352 rows and zero prompt overlap across base training, LLM pilot, development, final calibration, and test roles.

A frozen sentence encoder maps each prompt and response to a 768-dimensional feature vector. The base network predicts the five human ratings, and the nonconformity score is the maximum absolute residual, so the final set is a five-dimensional box controlled by one learned radius. The experiment compares supervised RCP, synthetic-only quantile learning, and PPI using two external LLM routes, denoted Flash and Pro. Each of 20 paired seeds uses 100 paired human/LLM ratings, 400 synthetic-only ratings in the equal-count arms, an independent 100-row development role, 400 human-rated final-calibration rows, and the official prompt-separated test split. Appendix~\ref{app:helpsteer} gives the complete protocol and leakage audit.

Figure~\ref{fig:helpsteer} and Table~\ref{tab:helpsteer_downstream} show the deployed results. Flash-PPI reduces grouped MSCE from $0.00665$ to $0.00502$ (24.6\%), and Pro-PPI reduces it to $0.00400$ (39.9\%). Equal-count Pro-PPI also raises minimum-group coverage from $0.738$ to $0.773$ and lowers mean log volume from $-0.426$ to $-0.481$. These gains are not a free model-ranking result: the Pro route costs 3.20 times as much as Flash at equal label count. At nearly equal spend, 400 Flash labels deliver slightly lower average MSCE and smaller sets than 54.8 Pro labels, whereas the cost-matched Pro arm has slightly higher minimum-group coverage. The experiment therefore exhibits a genuine label quality--quantity--cost tradeoff.

\begin{figure}[htbp]
    \centering
    \includegraphics[width=0.82\linewidth]{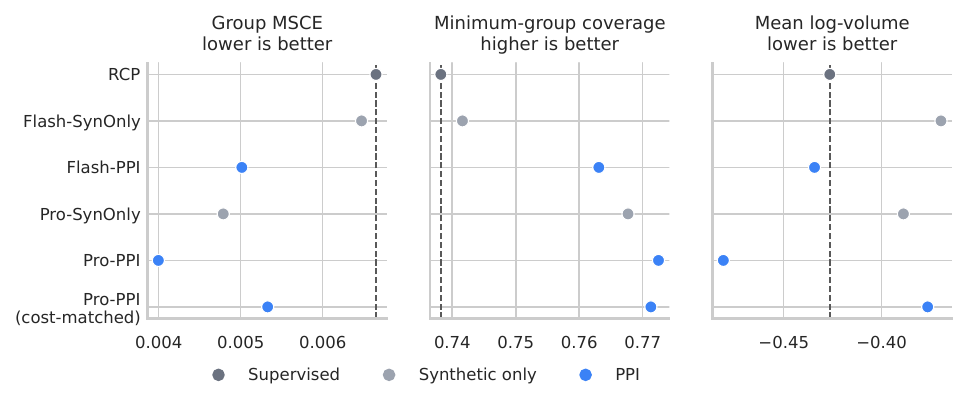}
    \caption{HelpSteer2 downstream results over 20 paired seeds. Equal-count PPI uses 400 auxiliary labels; the cost-matched Pro arm uses 54.8 on average. Dashed lines mark supervised RCP.}
    \label{fig:helpsteer}
\end{figure}
\begin{table}[htbp]
\centering
\caption{HelpSteer2 downstream results, averaged over 20 paired seeds. Equal-count arms use 100 paired points and 400 auxiliary labels; the cost-matched Pro arm uses 54.8 auxiliary labels on average. Lower MSCE, true pinball risk, and log-volume are better; higher minimum-group coverage is better. The two panels report the same methods and seeds.}
\label{tab:helpsteer_downstream}
\begin{tabular*}{\linewidth}{@{\extracolsep{\fill}}lrrrr@{}}
\toprule
Method & Aux. $N$ & Query cost & True pinball & Coverage \\
\midrule
RCP & 0.0 & \$0.0000 & 0.0558 & 0.8794 \\
Flash-SynOnly & 400.0 & \$0.0422 & 0.0520 & 0.8865 \\
Flash-PPI & 400.0 & \$0.0528 & 0.0611 & 0.8899 \\
Pro-SynOnly & 400.0 & \$0.1352 & 0.0507 & 0.8886 \\
Pro-PPI & 400.0 & \$0.1693 & 0.0619 & 0.8937 \\
Pro-PPI (cost-matched) & 54.8 & \$0.0527 & 0.0721 & 0.8846 \\
\bottomrule
\end{tabular*}

\tablepanelgap
\begin{tabular*}{\linewidth}{@{\extracolsep{\fill}}lrrr@{}}
\toprule
Method & MSCE & Min-group coverage & Log-volume \\
\midrule
RCP & 0.00665 & 0.738 & $-0.426$ \\
Flash-SynOnly & 0.00648 & 0.742 & $-0.370$ \\
Flash-PPI & 0.00502 & 0.763 & $-0.434$ \\
Pro-SynOnly & 0.00479 & 0.768 & $-0.389$ \\
Pro-PPI & 0.00400 & 0.773 & $-0.481$ \\
Pro-PPI (cost-matched) & 0.00533 & 0.771 & $-0.376$ \\
\bottomrule
\end{tabular*}
\end{table}

\begin{table}[htbp]
\centering
\caption{Paired HelpSteer2 effects relative to RCP. Entries are mean differences with paired-bootstrap 95\% intervals. Negative MSCE and log-volume and positive minimum-group coverage favor PPI.}
\label{tab:helpsteer_paired}
\setlength{\tabcolsep}{4pt}
\begin{tabularx}{\linewidth}{@{}>{\raggedright\arraybackslash}p{0.23\linewidth}*{3}{>{\centering\arraybackslash}X}@{}}
\toprule
Metric & Flash-PPI & Pro-PPI & \shortstack{Pro-PPI\\(cost-matched)} \\
\midrule
$\Delta$MSCE & \shortstack{$-0.00164$\\$[-0.00294,-0.00033]$} & \shortstack{$-0.00265$\\$[-0.00347,-0.00179]$} & \shortstack{$-0.00132$\\$[-0.00263,-0.00005]$} \\
\addlinespace[4pt]
\shortstack[l]{$\Delta$ minimum-group\\coverage} & \shortstack{$+0.0249$\\$[-0.0006,+0.0494]$} & \shortstack{$+0.0343$\\$[+0.0127,+0.0556]$} & \shortstack{$+0.0331$\\$[+0.0163,+0.0496]$} \\
\addlinespace[4pt]
$\Delta$ log-volume & \shortstack{$-0.0078$\\$[-0.0375,+0.0242]$} & \shortstack{$-0.0543$\\$[-0.0909,-0.0168]$} & \shortstack{$+0.0501$\\$[+0.0091,+0.0899]$} \\
\bottomrule
\end{tabularx}
\end{table}

\begin{table}[htbp]
\centering
\caption{Global rating accuracy and local boundary information on the 500-row HelpSteer2 pilot. Shuffling preserves each synthetic-score marginal while destroying input-specific alignment.}
\label{tab:helpsteer_boundary}
\setlength{\tabcolsep}{4pt}
\begin{tabular*}{\linewidth}{@{\extracolsep{\fill}}lrrrrr@{}}
\toprule
Labeler & Macro MAE & \shortstack{Observed\\disagree.} & \shortstack{Shuffled\\disagree.} & \shortstack{Reduction vs.\\shuffled} & \shortstack{Balanced\\acc.} \\
\midrule
Constant median & 0.680 & 0.136 & 0.136 & 0.0\% & 0.500 \\
Flash & 0.706 & 0.134 & 0.198 & 32.5\% & 0.637 \\
Pro & 0.846 & 0.152 & 0.251 & 39.5\% & 0.701 \\
\bottomrule
\end{tabular*}
\end{table}

Global rating error is not the relevant diagnostic by itself. Only 13.6\% of trusted pilot scores lie above the target boundary, so an always-below constant-median labeler achieves 86.4\% raw boundary agreement without useful discrimination. The shuffled-score reference in Table~\ref{tab:helpsteer_boundary} preserves each LLM's score marginal while destroying input-specific alignment. Flash and Pro reduce disagreement relative to that reference by 32.5\% and 39.5\%, with balanced boundary accuracies 0.637 and 0.701. Thus the external judges retain local boundary-side information even though neither beats the constant-median baseline in global five-rating MAE.

This case study also clarifies the profiled-risk perspective. The PPI models do not improve held-out raw or cross-calibrated pinball risk relative to RCP, yet their final grouped conditional-coverage error is lower. The observation is compatible with Proposition~\ref{prop:profiled_risk}: risk values and functional-gradient norms need not rank a finite candidate set identically. It is not presented as a direct empirical proof of the population identity, and prompt-cluster dependence makes the coverage results descriptive rather than a new row-level i.i.d. conformal guarantee.

\section{Discussion}\label{sec:discussion}
\paragraph{What synthetic labels contribute.}
The experiments support two complementary roles. Bias-corrected PPI uses paired trusted and synthetic scores as a control variate for the quantile-learning risk. Its first-order value is governed by post-conformal covariance, not by generic prediction accuracy. Flexible positive augmentation can additionally stabilize a high-variance threshold learner, but this benefit is not bias robust and need not reflect correct covariate--label pairing. The permutation experiment separates these effects rather than collapsing them into one claim.

\paragraph{Practical design lessons.}
First, preserve an untouched trusted final-calibration split: it insulates marginal validity from all upstream uses of synthetic supervision. Second, compare synthetic methods with the supervised procedure in the same conformal family; otherwise improvements may be confounded by different scores, architectures, or calibration allocations. Third, report conditional diagnostics together with marginal coverage and set size. A method that reduces a grouped error only by creating full or excessively wide sets is not useful. Finally, treat labeler choice as a task- and budget-dependent decision. The amount of boundary-aligned information per dollar may matter more than either global label accuracy or the nominal strength of the external model.

\paragraph{Limitations.}
The sharp three-resource theorem is local and fixed dimensional, with the neural implementations interpreted through controlled comparisons and complementary consistency results in the appendix. Grouped MSCE, WSC, and $L_1$-ERT are finite-resolution diagnostics, not distribution-free pointwise conditional-validity guarantees. The main tabular results concern eight regression datasets, one target coverage level, and fixed external labelers. The HelpSteer2 study uses prompt clusters and frozen LLM outputs, so it should be read as a realistic case study rather than a universal guarantee about LLM-generated labels. Finally, synthetic labels do not replace trusted outcomes: they cannot remove final-calibration uncertainty, and equal-quality additional trusted labels remain a strong resource benchmark.

\section{Conclusion}
With abundant covariates but few trusted labels, prediction-powered quantile learning offers a principled route from synthetic supervision to more reliable conditional coverage while preserving the usual marginal guarantee. Across external labelers---including domain models and general-purpose LLMs---bias-corrected quantile learning improves deployed coverage and often yields smaller prediction sets. The theory explains why these gains depend neither on global label accuracy nor on raw learning-stage covariance: final conformalization removes global threshold shifts and retains only density-weighted shape information aligned with the component of trusted learning error visible at deployment. This post-conformal view separates usable synthetic information from synthetic noise and finite-calibration error, and provides a benefit--cost rule for how strongly synthetic labels should influence quantile learning.

\bibliography{ref}

@inproceedings{plassier2025rectifying,
  title     = {Rectifying Conformity Scores for Better Conditional Coverage},
  author    = {Plassier, Vincent and Fishkov, Alexander and Dheur, Victor and Guizani, Mohsen and Ben Taieb, Souhaib and Panov, Maxim and Moulines, Eric},
  booktitle = {Proceedings of the 42nd International Conference on Machine Learning},
  series    = {Proceedings of Machine Learning Research},
  volume    = {267},
  pages     = {49459--49492},
  publisher = {PMLR},
  year      = {2025}
}

@inproceedings{kiyani2024conformal,
  title     = {Conformal Prediction with Learned Features},
  author    = {Kiyani, Shayan and Pappas, George J. and Hassani, Hamed},
  booktitle = {Proceedings of the 41st International Conference on Machine Learning},
  series    = {Proceedings of Machine Learning Research},
  volume    = {235},
  pages     = {24749--24769},
  publisher = {PMLR},
  year      = {2024}
}

@article{lei2018distribution,
  title={Distribution-free predictive inference for regression},
  author={Lei, Jing and G’Sell, Max and Rinaldo, Alessandro and Tibshirani, Ryan J and Wasserman, Larry},
  journal={Journal of the American Statistical Association},
  volume={113},
  number={523},
  pages={1094--1111},
  year={2018},
  publisher={Taylor \& Francis}
}

@article{lei2014distribution,
  title={Distribution-free prediction bands for non-parametric regression},
  author={Lei, Jing and Wasserman, Larry},
  journal={Journal of the Royal Statistical Society Series B: Statistical Methodology},
  volume={76},
  number={1},
  pages={71--96},
  year={2014},
  publisher={Oxford University Press}
}

@book{vovk2005algorithmic,
  title={Algorithmic learning in a random world},
  author={Vovk, Vladimir and Gammerman, Alexander and Shafer, Glenn},
  year={2005},
  publisher={Springer}
}

@article{shafer2008tutorial,
  title   = {A Tutorial on Conformal Prediction},
  author  = {Shafer, Glenn and Vovk, Vladimir},
  journal = {Journal of Machine Learning Research},
  volume  = {9},
  number  = {12},
  pages   = {371--421},
  year    = {2008}
}

@article{hore2024local,
  title   = {Conformal Prediction with Local Weights: Randomization Enables Robust Guarantees},
  author  = {Hore, Rohan and Barber, Rina Foygel},
  journal = {Journal of the Royal Statistical Society Series B: Statistical Methodology},
  volume  = {87},
  number  = {2},
  pages   = {549--578},
  year    = {2025},
  doi     = {10.1093/jrsssb/qkae103}
}

@inproceedings{vovk2012conditional,
  title     = {Conditional Validity of Inductive Conformal Predictors},
  author    = {Vovk, Vladimir},
  booktitle = {Proceedings of the Asian Conference on Machine Learning},
  series    = {Proceedings of Machine Learning Research},
  volume    = {25},
  pages     = {475--490},
  publisher = {PMLR},
  year      = {2012}
}

@article{foygel2021limits,
  title={The limits of distribution-free conditional predictive inference},
  author={Foygel Barber, Rina and Candes, Emmanuel J and Ramdas, Aaditya and Tibshirani, Ryan J},
  journal={Information and Inference: A Journal of the IMA},
  volume={10},
  number={2},
  pages={455--482},
  year={2021},
  publisher={Oxford University Press}
}

@article{gibbs2025conformal,
  title   = {Conformal Prediction with Conditional Guarantees},
  author  = {Gibbs, Isaac and Cherian, John J. and Cand{\`e}s, Emmanuel J.},
  journal = {Journal of the Royal Statistical Society Series B: Statistical Methodology},
  volume  = {87},
  number  = {4},
  pages   = {1100--1126},
  year    = {2025},
  doi     = {10.1093/jrsssb/qkaf008}
}

@inproceedings{jung2023batch,
title={Batch Multivalid Conformal Prediction},
author={Christopher Jung and Georgy Noarov and Ramya Ramalingam and Aaron Roth},
booktitle={International Conference on Learning Representations},
year={2023},
}

@article{ding2023class,
  title={Class-conditional conformal prediction with many classes},
  author={Ding, Tiffany and Angelopoulos, Anastasios and Bates, Stephen and Jordan, Michael and Tibshirani, Ryan J},
  journal={Advances in neural information processing systems},
  volume={36},
  pages={64555--64576},
  year={2023}
}

@article{bian2023training,
  title={Training-conditional coverage for distribution-free predictive inference},
  author={Bian, Michael and Barber, Rina Foygel},
  journal={Electronic Journal of Statistics},
  volume={17},
  number={2},
  pages={2044--2066},
  year={2023},
  publisher={The Institute of Mathematical Statistics and the Bernoulli Society}
}

@article{romano2019conformalized,
  title={Conformalized quantile regression},
  author={Romano, Yaniv and Patterson, Evan and Candes, Emmanuel},
  journal={Advances in neural information processing systems},
  volume={32},
  year={2019}
}

@article{izbicki2022cd,
  title={Cd-split and hpd-split: Efficient conformal regions in high dimensions},
  author={Izbicki, Rafael and Shimizu, Gilson and Stern, Rafael B},
  journal={Journal of Machine Learning Research},
  volume={23},
  number={87},
  pages={1--32},
  year={2022}
}

@article{xie2024boosted,
  title={Boosted conformal prediction intervals},
  author={Xie, Ran and Barber, Rina and Candes, Emmanuel},
  journal={Advances in Neural Information Processing Systems},
  volume={37},
  pages={71868--71899},
  year={2024}
}

@misc{braun2025conditionalcoveragediagnosticsconformal,
      title={Conditional Coverage Diagnostics for Conformal Prediction}, 
      author={Sacha Braun and David Holzmüller and Michael I. Jordan and Francis Bach},
      year={2025},
      eprint={2512.11779},
      archivePrefix={arXiv},
}

@inproceedings{plassier2025probabilistic,
title={Probabilistic Conformal Prediction with Approximate Conditional Validity},
author={Vincent Plassier and Alexander Fishkov and Mohsen Guizani and Maxim Panov and Eric Moulines},
booktitle={International Conference on Learning Representations},
year={2025},
}

@article{cauchois2021knowing,
  title={Knowing what you know: valid and validated confidence sets in multiclass and multilabel prediction},
  author={Cauchois, Maxime and Gupta, Suyash and Duchi, John C},
  journal={Journal of machine learning research},
  volume={22},
  number={81},
  pages={1--42},
  year={2021}
}

@article{chen2025colorful,
  title={Colorful Pinball: Density-Weighted Quantile Regression for Conditional Guarantee of Conformal Prediction},
  author={Chen, Qianyi and Li, Bo},
  journal={arXiv preprint arXiv:2512.24139},
  year={2025}
}

@article{angelopoulos2023prediction,
  title   = {Prediction-Powered Inference},
  author  = {Angelopoulos, Anastasios N. and Bates, Stephen and Fannjiang, Clara and Jordan, Michael I. and Zrnic, Tijana},
  journal = {Science},
  volume  = {382},
  number  = {6671},
  pages   = {669--674},
  year    = {2023},
  doi     = {10.1126/science.adi6000}
}

@article{angelopoulos2023ppi++,
  title={Ppi++: Efficient prediction-powered inference},
  author={Angelopoulos, Anastasios N and Duchi, John C and Zrnic, Tijana},
  journal={arXiv preprint arXiv:2311.01453},
  year={2023}
}

@article{ji2025predictions,
  title={Predictions as Surrogates: Revisiting Surrogate Outcomes in the Age of AI},
  author={Ji, Wenlong and Lei, Lihua and Zrnic, Tijana},
  journal={arXiv preprint arXiv:2501.09731},
  year={2025}
}

@inproceedings{bashari2025syntheticpowered,
  title     = {Synthetic-Powered Predictive Inference},
  author    = {Bashari, Meshi and Lotan, Roy Maor and Lee, Yonghoon and Dobriban, Edgar and Romano, Yaniv},
  booktitle = {Advances in Neural Information Processing Systems},
  volume    = {38},
  year      = {2025},
  doi       = {10.52202/085713-3925}
}

@article{zhou2025semi,
  title={Semi-Supervised Conformal Prediction With Unlabeled Nonconformity Score},
  author={Zhou, Xuanning and Zeng, Hao and Xia, Xiaobo and Jing, Bingyi and Wei, Hongxin},
  journal={arXiv preprint arXiv:2505.21147},
  year={2025}
}

@inproceedings{shoham2025predictionpowered,
  title     = {Prediction-Powered Semi-Supervised Learning with Online Power Tuning},
  author    = {Shoham, Noa and Dorfman, Ron and Shaer, Shalev and Levy, Kfir Yehuda and Romano, Yaniv},
  booktitle = {Advances in Neural Information Processing Systems},
  volume    = {38},
  year      = {2025},
  doi       = {10.52202/085713-3632}
}

@book{chapelle2006ssl,
  title     = {Semi-Supervised Learning},
  editor    = {Chapelle, Olivier and Sch{\"o}lkopf, Bernhard and Zien, Alexander},
  publisher = {MIT Press},
  address   = {Cambridge, MA},
  year      = {2006}
}

@inproceedings{colombo2024normalizing,
  title     = {Normalizing Flows for Conformal Regression},
  author    = {Colombo, Nicol{\`o}},
  booktitle = {Proceedings of the 40th Conference on Uncertainty in Artificial Intelligence},
  series    = {Proceedings of Machine Learning Research},
  volume    = {244},
  pages     = {881--893},
  publisher = {PMLR},
  year      = {2024}
}

@inproceedings{seedat2023improving,
  title     = {Improving Adaptive Conformal Prediction Using Self-Supervised Learning},
  author    = {Seedat, Nabeel and Jeffares, Alan and Imrie, Fergus and van der Schaar, Mihaela},
  booktitle = {Proceedings of the 26th International Conference on Artificial Intelligence and Statistics},
  series    = {Proceedings of Machine Learning Research},
  volume    = {206},
  pages     = {10160--10177},
  publisher = {PMLR},
  year      = {2023}
}
\bibliographystyle{iclr2027_conference}

\appendix

\section{Experimental Details and Full Results}\label{app:experiments}
\subsection{Datasets and role protocol}\label{app:datasets}
The eight datasets are Bike Sharing, Diamonds, Gas Turbine, Naval Propulsion, SGEMM, Superconductivity, Transcoding, and WEC. They span scalar and multivariate regression, with response dimension up to 49. The package records the source, version, preprocessing, and license of every dataset. All cleaned observations are used except for a fixed 50,000-row SGEMM sample.

Each seed begins with one random permutation and non-overlapping maximum reservoirs: 1\% preprocessing, 16\% auxiliary-labeler training, 4\% labeler validation, up to 40\% synthetic-pool covariates, up to 4\% base/CQR training, up to 2\% conformal data, 3\% grouping, 3\% worst-slice selection, and the remaining observations for evaluation. The main experiment uses a 30\% synthetic pool, 2\% base training, and a 1\% conformal reservoir. Resource sweeps use nested prefixes, so changing one resource does not move the evaluation data or any other role. Ordinary synthetic methods never access true outcomes in the synthetic pool.

\subsection{Detailed baseline definitions}\label{app:method_details}
\paragraph{Split, SPPI, and NNM.}
Split CP calibrates the global $\ell_\infty$ residual score on the complete conformal reservoir. SPPI~\cite{bashari2025syntheticpowered} uses the published rank-based transporter between trusted and synthetic score distributions. NNM~\cite{zhou2025semi} matches synthetic observations to trusted neighbors and corrects their scores before calibration. All three share the same base predictor, data roles, and evaluation rows.

\paragraph{CQR and CQR-PPI.}
CQR~\cite{romano2019conformalized} predicts coordinatewise lower and upper conditional outcome quantiles with a two-hidden-layer network and conformalizes the maximum signed endpoint error. CQR-PPI replaces each endpoint pinball risk by \eqref{eq:ppi_obj_main}, using paired and pool synthetic outcomes; both methods run on all eight datasets, WEC included.

\paragraph{RCP family.}
RCP learns the conditional $\tau$-quantile of the $\ell_\infty$ residual score on one half of the conformal reservoir and uses the other half for final calibration. RCP-PTFT pretrains the same network on synthetic scores and fine-tunes on trusted scores. RCP-PPI-CV chooses $\lambda$ from $\{0,.25,.5,.75,1\}$ by five-fold trusted pinball validation and refits on the full threshold-learning half. RCP-Aug minimizes $P_n\ell_q(X,S)+0.5\widetilde P_N\ell_q(\widetilde X,\widetilde S)$; because it has no paired subtraction, it is not bias robust. RCP-PPI trains the same network with \eqref{eq:ppi_obj_main} at $\lambda=1$. All RCP variants use bounded nonnegative outputs, Adam learning rate $2\times10^{-3}$, 100 epochs, weight decay, gradient clipping, and common initialization. CQR uses the same width-128 architecture for 200 epochs.

\paragraph{Ablation controls.}
RCP-Aug-Permuted permutes synthetic scores across pool covariates. RCP-Extra uses true pool outcomes and is explicitly a resource upper control. RCP-Pool assigns equal positive weight to every trusted and synthetic score and represents stronger trust in the synthetic law. None of these controls appears in the ten-method ranking.

\subsection{Metrics and statistical reporting}\label{app:metric_details}
K-means with $K=30$ is fitted only on the independent grouping sample. For held-out group coverage $\widehat c_k$ and group mass $\widehat p_k$, the main table reports
\begin{equation}\label{eq:group_raw_msce}
\widehat{\MSCE}_{30}=\sum_{k=1}^{30}\widehat p_k(\widehat c_k-\tau)^2.
\end{equation}
This nonnegative estimator contains finite-evaluation Bernoulli noise. We therefore repeat the analysis with $K=10$, split the evaluation rows in each group into independent halves $A_k,B_k$, and report
\begin{equation}\label{eq:group_cross_product}
\widehat{\MSCE}_{\rm split}=\sum_{k=1}^{10}\widehat p_k(\widehat c_{A_k}-\tau)(\widehat c_{B_k}-\tau),
\end{equation}
whose expectation removes the leading Bernoulli term conditional on the groups. WSC searches 64 fixed projection directions on a dedicated selection sample, subject to minimum mass 0.1, and evaluates the selected slab on different observations. $L_1$-ERT uses an unweighted cross-fitted logistic coverage model against the constant nominal-coverage predictor. We also report marginal coverage, original-scale log volume, empty/full-set frequencies, failures, and true-outcome access.

Every main result uses 30 common data splits. Paired effects subtract the supervised method in the same family on the same split and evaluation observations. Percentile intervals use 5,000 bootstrap resamples of the within-split difference. They are descriptive, dataset-specific intervals without multiple-comparison adjustment; no cross-dataset effect is formed.

\subsection{Paired results, marginal coverage, and volume}\label{app:completed_results}
The main text reports complete MSCE, WSC, and $L_1$-ERT tables; this section adds paired grouped-MSCE changes, marginal coverage, mean log volume, and an independent split-half audit.
\begin{table}[htbp]
\centering
\caption{Paired change in grouped MSCE relative to the supervised method in the same family (synthetic minus supervised, $\times10^3$). Negative is better; $\dagger$ marks a descriptive paired-bootstrap 95\% interval entirely below zero. (Bike, Diamonds, Gas Turbine, and Naval.)}
\label{tab:paired_msce_appendix}
\setlength{\tabcolsep}{4.5pt}
\begin{tabular*}{\linewidth}{@{\extracolsep{\fill}}lrrrr@{}}
\toprule
Method & Bike & Diamond & Gas Turbine & Naval \\
\midrule
SPPI & +1.03 & +0.91 & +1.92 & +1.84 \\
NNM & +0.24 & -0.11 & +0.03 & +0.87 \\
CQR-PPI & -0.79$^\dagger$ & -2.69$^\dagger$ & -0.30$^\dagger$ & -1.93$^\dagger$ \\
RCP-PTFT & -0.78$^\dagger$ & -1.53$^\dagger$ & -0.99$^\dagger$ & -1.02$^\dagger$ \\
RCP-PPI-CV & -2.86$^\dagger$ & -2.77$^\dagger$ & -0.78$^\dagger$ & -2.65$^\dagger$ \\
RCP-Aug & -2.36$^\dagger$ & -3.45$^\dagger$ & -1.21$^\dagger$ & -2.69$^\dagger$ \\
\textbf{RCP-PPI} & -3.24$^\dagger$ & -3.79$^\dagger$ & -1.41$^\dagger$ & -3.88$^\dagger$ \\
\bottomrule
\end{tabular*}
\end{table}

\begin{table}[htbp]
\centering
\caption{Paired change in grouped MSCE relative to the supervised method in the same family (synthetic minus supervised, $\times10^3$). Negative is better; $\dagger$ marks a descriptive paired-bootstrap 95\% interval entirely below zero. (SGEMM, Superconductivity, Transcoding, and WEC.)}
\label{tab:paired_msce_appendix:second}
\setlength{\tabcolsep}{4.5pt}
\begin{tabular*}{\linewidth}{@{\extracolsep{\fill}}lrrrr@{}}
\toprule
Method & SGEMM & Supercond. & Transcoding & WEC \\
\midrule
SPPI & +0.22 & +2.93 & +0.70 & +2.70 \\
NNM & -0.07 & +0.89 & -0.05 & -0.15 \\
CQR-PPI & -2.85$^\dagger$ & -1.61$^\dagger$ & -1.53$^\dagger$ & +1.01 \\
RCP-PTFT & -0.51$^\dagger$ & -0.55 & -0.56$^\dagger$ & -1.33$^\dagger$ \\
RCP-PPI-CV & -0.90 & -2.73$^\dagger$ & -0.97$^\dagger$ & +1.17 \\
RCP-Aug & -1.42$^\dagger$ & -2.65$^\dagger$ & -1.25$^\dagger$ & -0.79 \\
\textbf{RCP-PPI} & -0.59 & -2.95$^\dagger$ & -1.55$^\dagger$ & +1.29 \\
\bottomrule
\end{tabular*}
\end{table}

\begin{table}[htbp]
\centering
\caption{Marginal coverage in the main benchmark (mean $\pm$ standard deviation over 30 splits). The target is 0.9. (Bike, Diamonds, Gas Turbine, and Naval.)}
\label{tab:marginal_coverage_appendix}
\setlength{\tabcolsep}{4.5pt}
\begin{tabular*}{\linewidth}{@{\extracolsep{\fill}}lrrrr@{}}
\toprule
Method & Bike & Diamond & Gas Turbine & Naval \\
\midrule
Split & 0.909$\pm$0.023 & 0.900$\pm$0.013 & 0.898$\pm$0.013 & 0.909$\pm$0.023 \\
SPPI & 0.892$\pm$0.016 & 0.894$\pm$0.008 & 0.882$\pm$0.014 & 0.887$\pm$0.022 \\
NNM & 0.904$\pm$0.019 & 0.900$\pm$0.011 & 0.899$\pm$0.015 & 0.900$\pm$0.019 \\
\midrule
CQR & 0.904$\pm$0.022 & 0.897$\pm$0.011 & 0.900$\pm$0.017 & 0.906$\pm$0.026 \\
CQR-PPI & 0.901$\pm$0.030 & 0.898$\pm$0.012 & 0.901$\pm$0.015 & 0.901$\pm$0.028 \\
\midrule
RCP & 0.910$\pm$0.043 & 0.905$\pm$0.018 & 0.904$\pm$0.024 & 0.916$\pm$0.030 \\
RCP-PTFT & 0.910$\pm$0.039 & 0.903$\pm$0.019 & 0.904$\pm$0.022 & 0.912$\pm$0.031 \\
RCP-PPI-CV & 0.911$\pm$0.036 & 0.906$\pm$0.016 & 0.906$\pm$0.023 & 0.919$\pm$0.031 \\
RCP-Aug & 0.912$\pm$0.036 & 0.905$\pm$0.016 & 0.905$\pm$0.025 & 0.914$\pm$0.031 \\
\textbf{RCP-PPI} & 0.910$\pm$0.035 & 0.902$\pm$0.014 & 0.907$\pm$0.023 & 0.915$\pm$0.030 \\
\bottomrule
\end{tabular*}
\end{table}

\begin{table}[htbp]
\centering
\caption{Marginal coverage in the main benchmark (mean $\pm$ standard deviation over 30 splits). The target is 0.9. (SGEMM, Superconductivity, Transcoding, and WEC.)}
\label{tab:marginal_coverage_appendix:second}
\setlength{\tabcolsep}{4.5pt}
\begin{tabular*}{\linewidth}{@{\extracolsep{\fill}}lrrrr@{}}
\toprule
Method & SGEMM & Supercond. & Transcoding & WEC \\
\midrule
Split & 0.901$\pm$0.013 & 0.907$\pm$0.017 & 0.897$\pm$0.013 & 0.895$\pm$0.016 \\
SPPI & 0.896$\pm$0.009 & 0.890$\pm$0.022 & 0.890$\pm$0.011 & 0.882$\pm$0.016 \\
NNM & 0.901$\pm$0.008 & 0.902$\pm$0.022 & 0.896$\pm$0.009 & 0.895$\pm$0.017 \\
\midrule
CQR & 0.903$\pm$0.013 & 0.899$\pm$0.020 & 0.902$\pm$0.011 & 0.895$\pm$0.016 \\
CQR-PPI & 0.899$\pm$0.012 & 0.902$\pm$0.017 & 0.901$\pm$0.013 & 0.894$\pm$0.016 \\
\midrule
RCP & 0.898$\pm$0.016 & 0.905$\pm$0.025 & 0.900$\pm$0.017 & 0.899$\pm$0.028 \\
RCP-PTFT & 0.900$\pm$0.017 & 0.902$\pm$0.029 & 0.901$\pm$0.015 & 0.901$\pm$0.025 \\
RCP-PPI-CV & 0.896$\pm$0.019 & 0.904$\pm$0.028 & 0.900$\pm$0.015 & 0.900$\pm$0.030 \\
RCP-Aug & 0.900$\pm$0.017 & 0.908$\pm$0.025 & 0.901$\pm$0.016 & 0.898$\pm$0.028 \\
\textbf{RCP-PPI} & 0.894$\pm$0.020 & 0.908$\pm$0.026 & 0.902$\pm$0.017 & 0.901$\pm$0.026 \\
\bottomrule
\end{tabular*}
\end{table}

\begin{table}[htbp]
\centering
\caption{Mean log volume on the original outcome scale (mean $\pm$ standard deviation over 30 splits). Values are comparable only within a dataset; no method produced a full set. (Bike, Diamonds, Gas Turbine, and Naval.)}
\label{tab:log_volume_appendix}
\setlength{\tabcolsep}{4.5pt}
\begin{tabular*}{\linewidth}{@{\extracolsep{\fill}}lrrrr@{}}
\toprule
Method & Bike & Diamond & Gas Turbine & Naval \\
\midrule
Split & 6.16$\pm$0.10 & 7.64$\pm$0.09 & 4.35$\pm$0.14 & -7.07$\pm$0.17 \\
SPPI & 6.08$\pm$0.08 & 7.60$\pm$0.06 & 4.21$\pm$0.13 & -7.18$\pm$0.16 \\
NNM & 6.13$\pm$0.11 & 7.65$\pm$0.08 & 4.36$\pm$0.17 & -7.13$\pm$0.17 \\
\midrule
CQR & 5.89$\pm$0.10 & 7.25$\pm$0.06 & 4.18$\pm$0.13 & -6.88$\pm$0.08 \\
CQR-PPI & 5.70$\pm$0.07 & 7.09$\pm$0.06 & 4.15$\pm$0.12 & -6.90$\pm$0.09 \\
\midrule
RCP & 6.13$\pm$0.18 & 7.36$\pm$0.11 & 4.39$\pm$0.23 & -7.04$\pm$0.19 \\
RCP-PTFT & 6.10$\pm$0.15 & 7.30$\pm$0.10 & 4.34$\pm$0.19 & -7.07$\pm$0.20 \\
RCP-PPI-CV & 6.07$\pm$0.14 & 7.29$\pm$0.11 & 4.38$\pm$0.22 & -7.05$\pm$0.22 \\
RCP-Aug & 6.08$\pm$0.14 & 7.27$\pm$0.09 & 4.32$\pm$0.21 & -7.08$\pm$0.20 \\
\textbf{RCP-PPI} & 6.05$\pm$0.12 & 7.24$\pm$0.06 & 4.37$\pm$0.22 & -7.09$\pm$0.19 \\
\bottomrule
\end{tabular*}
\end{table}

\begin{table}[htbp]
\centering
\caption{Mean log volume on the original outcome scale (mean $\pm$ standard deviation over 30 splits). Values are comparable only within a dataset; no method produced a full set. (SGEMM, Superconductivity, Transcoding, and WEC.)}
\label{tab:log_volume_appendix:second}
\setlength{\tabcolsep}{4.5pt}
\begin{tabular*}{\linewidth}{@{\extracolsep{\fill}}lrrrr@{}}
\toprule
Method & SGEMM & Supercond. & Transcoding & WEC \\
\midrule
Split & 22.46$\pm$0.33 & 4.06$\pm$0.09 & 14.61$\pm$0.16 & 536.00$\pm$1.93 \\
SPPI & 22.33$\pm$0.24 & 3.98$\pm$0.11 & 14.51$\pm$0.13 & 534.63$\pm$1.95 \\
NNM & 22.44$\pm$0.24 & 4.04$\pm$0.11 & 14.59$\pm$0.10 & 536.01$\pm$2.09 \\
\midrule
CQR & 21.62$\pm$0.33 & 3.68$\pm$0.08 & 13.25$\pm$0.17 & 519.16$\pm$2.59 \\
CQR-PPI & 20.09$\pm$0.24 & 3.53$\pm$0.06 & 12.71$\pm$0.14 & 516.41$\pm$2.25 \\
\midrule
RCP & 21.92$\pm$0.35 & 3.83$\pm$0.13 & 14.22$\pm$0.18 & 530.73$\pm$3.70 \\
RCP-PTFT & 21.85$\pm$0.36 & 3.78$\pm$0.13 & 14.16$\pm$0.16 & 528.19$\pm$3.28 \\
RCP-PPI-CV & 21.72$\pm$0.32 & 3.74$\pm$0.12 & 14.09$\pm$0.15 & 525.83$\pm$3.51 \\
RCP-Aug & 21.78$\pm$0.33 & 3.77$\pm$0.12 & 14.08$\pm$0.14 & 526.05$\pm$3.51 \\
\textbf{RCP-PPI} & 21.66$\pm$0.33 & 3.74$\pm$0.11 & 14.04$\pm$0.13 & 525.53$\pm$3.34 \\
\bottomrule
\end{tabular*}
\end{table}

The $K=10$ raw, debiased, and split-half audits agree with the principal $K=30$ direction in nearly all dataset--method cells. Complete interval endpoints, win rates, empty-set frequencies, and task-level audit fields are distributed with the results package.

\subsection{RCP interpretation controls}\label{app:ablation_results}
Direct pooling gives low grouped error on several datasets but assumes that synthetic labels can be trusted without bias correction. RCP-Extra uses true pool outcomes and confirms that additional trusted labels remain a strong resource control. Figure~\ref{fig:pairing_ablation} in the main text visualizes the matched-permutation contrast; the precise values are recorded below.
\begin{table}[htbp]
\centering
\caption{Genuine RCP augmentation versus a matched permutation placebo. The placebo preserves the synthetic-score marginal distribution, sample size, architecture, and optimization, but breaks the covariate--score pairing. Entries are paired differences (genuine minus permuted) with 95\% bootstrap intervals over 20 common splits. Negative MSCE and positive WSC favor genuine pairing.}
\label{tab:pairing_ablation}
\setlength{\tabcolsep}{3.5pt}
\begin{tabular*}{\linewidth}{@{\extracolsep{\fill}}lccc@{}}
\toprule
Dataset & $\Delta$MSCE$_{30}$ ($\times10^3$) & $\Delta$MSCE$_{10}$ ($\times10^3$) & $\Delta$WSC ($\times100$) \\
\midrule
Bike & -1.74 [-2.75, -0.88] & -1.57 [-2.65, -0.72] & 1.80 [0.20, 3.40] \\
Diamond & -4.73 [-5.78, -3.77] & -4.08 [-5.04, -3.14] & 10.41 [7.89, 12.78] \\
Naval & -1.13 [-2.07, -0.34] & -1.06 [-1.86, -0.37] & 1.93 [0.11, 3.92] \\
SGEMM & -1.58 [-2.11, -0.94] & -1.07 [-1.53, -0.55] & 4.08 [2.91, 5.33] \\
Transcoding & -1.14 [-1.56, -0.73] & -0.65 [-0.93, -0.40] & 2.91 [1.76, 4.28] \\
\bottomrule
\end{tabular*}
\end{table}

\subsection{Sensitivity to label quality and trusted-label budget}\label{app:sensitivity_tables}
The main text reports the weak/moderate/oracle labeler comparison and the trusted-label budget sweep. All cells use common splits, and budget sweeps vary only nested trusted-data prefixes while keeping every other role fixed.

\subsection{Supplementary experiments}\label{app:supp_experiments}
\paragraph{Pinball risk of the score-quantile learner.}
Proposition~\ref{prop:risk_coverage_bridge} motivates an intermediate check between the optimization objective and conditional coverage. We evaluate the supervised and unit-PPI score-quantile learners on the same trusted validation folds used for power selection. Table~\ref{tab:pinball_comparison} shows lower mean pinball risk under PPI-ERM on every dataset, together with lower across-split variability. This verifies that prediction-powered learning improves the quantile learner before we ask whether the gain survives final conformalization.
\begin{table}[htbp]
\centering
\caption{Intermediate pinball-risk comparison for the RCP score-quantile learner. ERM and unit PPI-ERM are evaluated on the same trusted validation folds used for power selection. We report mean $\pm$ standard deviation over 30 common splits; a positive gain means lower validation risk under PPI-ERM.}
\label{tab:pinball_comparison}
\setlength{\tabcolsep}{3.5pt}
\begin{tabular*}{\linewidth}{@{\extracolsep{\fill}}lcccc@{}}
\toprule
Dataset & ERM risk & PPI-ERM risk & Risk reduction & PPI-ERM win rate \\
\midrule
Bike & 0.1471$\pm$0.0585 & 0.1030$\pm$0.0364 & 0.0441$\pm$0.0348 & 0.93 \\
Diamond & 0.0338$\pm$0.0169 & 0.0308$\pm$0.0152 & 0.0030$\pm$0.0057 & 0.63 \\
Gas Turbine & 0.1006$\pm$0.0405 & 0.0975$\pm$0.0375 & 0.0030$\pm$0.0110 & 0.57 \\
Naval & 0.0930$\pm$0.0615 & 0.0822$\pm$0.0301 & 0.0108$\pm$0.0357 & 0.60 \\
SGEMM & 0.0493$\pm$0.0166 & 0.0434$\pm$0.0132 & 0.0059$\pm$0.0065 & 0.73 \\
Supercond. & 0.0885$\pm$0.0417 & 0.0754$\pm$0.0356 & 0.0132$\pm$0.0164 & 0.77 \\
Transcoding & 0.0794$\pm$0.0268 & 0.0762$\pm$0.0265 & 0.0032$\pm$0.0079 & 0.57 \\
WEC & 0.2350$\pm$0.0847 & 0.1649$\pm$0.0473 & 0.0702$\pm$0.0623 & 0.90 \\
\bottomrule
\end{tabular*}
\end{table}

\paragraph{Effect of the synthetic-pool size.}
The main experiments use a large synthetic pool because auxiliary labels are cheaper than trusted outcomes. Table~\ref{tab:varying_N} varies only the pool prefix while keeping trusted learning, final calibration, grouping, and evaluation samples fixed. Relative to supervised RCP, even a 5\% pool improves all three metrics on the four audited datasets. Further increases help until the pool-estimation component becomes small, after which performance may plateau or become nonmonotone because the usable covariance and power remain unchanged.
Table~\ref{tab:varying_N} in the main text reports the complete pool-size sweep.

\subsection{Realistic LLM-generated supervision on HelpSteer2}\label{app:helpsteer}
After removing exact normalized prompt--response duplicates and one prompt shared between the official train and validation splits, the experiment contains 21,352 rows. The official validation split is held out as test data, and every other role is allocated at the normalized-prompt level, giving zero prompt overlap across base training, LLM pilot, development, final calibration, and test roles. Each response receives five 0--4 human ratings for helpfulness, correctness, coherence, complexity, and verbosity as distinct quality dimensions.

Prompt and response are encoded separately by a frozen \texttt{all-MiniLM-L6-v2} model and concatenated into a 768-dimensional feature vector. The five-output base predictor is a $768\!\to\!128\!\to\!64\!\to\!5$ MLP trained on 801 rows. Its nonconformity score is $S(x,y)=\max_{k\le 5}|y_k-\widehat\mu_k(x)|$, so one scalar radius defines a five-dimensional box. The score-quantile learner is a $768\!\to\!64\!\to\!32\!\to\!1$ MLP. Within each seed, all methods share initialization. Checkpoint selection uses a separate 100-row human-rated development role, and 400 further human-rated rows are reserved for final conformalization.

The external labelers are the production \texttt{deepseek-v4-flash} and \texttt{deepseek-v4-pro} routes. They score the same 500 pilot rows at temperature zero in JSON mode, using the five HelpSteer2 rubrics and no in-context examples. All 1,000 calls return valid ratings without retries. For each seed, 100 pilot rows form the paired sample and the remaining 400 form the auxiliary pool. Synthetic-only methods fit only the auxiliary synthetic scores; PPI adds the signed true-minus-synthetic correction on the paired rows. The cost-matched Pro arm keeps the 100 paired calls and uses as many auxiliary Pro calls as fit within the Flash-PPI spend.

Figure~\ref{fig:helpsteer} and Tables~\ref{tab:helpsteer_downstream}--\ref{tab:helpsteer_boundary} in the main text report the deployed metrics and boundary audit. For the boundary audit, only 13.6\% of true pilot scores lie above the target boundary, so an always-below constant-median score obtains 86.4\% raw agreement without useful discrimination. We therefore compare each LLM with a shuffled-score reference that preserves its score marginal and above-boundary rate while breaking input-specific alignment. Table~\ref{tab:helpsteer_boundary} reports this comparison. The package includes the executed notebook, rendered report, paper-ready setting, and deterministic figure data.

\subsection{Theory diagnostics}\label{app:theory_diagnostics}
Table~\ref{tab:theory_experiment_map} links each principal controlled experiment to a prediction of the post-conformal analysis. These simulations isolate geometry, information mechanisms, and resource scaling; practical evidence comes from the real-data benchmark and permutation experiment.

\begin{table}[htbp]
\centering
\caption{Theory-to-experiment map for the controlled diagnostics.}
\label{tab:theory_experiment_map}
\setlength{\tabcolsep}{4.0pt}
\begin{tabular}{>{\raggedright\arraybackslash}p{0.225\linewidth}>{\raggedright\arraybackslash}p{0.255\linewidth}>{\raggedright\arraybackslash}p{0.445\linewidth}}
\toprule
Prediction & Experiment & Finding \\
\midrule
Global shifts and some raw covariance are conformally null & Exact shift check and conformal-null simulation & Sets are unchanged by a constant threshold shift; raw covariance remains positive while projected covariance is zero to numerical precision. \\
\addlinespace
The deployment map is a density-weighted first derivative with a second-order remainder & Population derivative check (Figure~\ref{fig:post_conformal_derivative}) & Constant perturbations vanish to machine precision; linear and quadratic remainder slopes are 1.981 and 2.009. \\
\addlinespace
$b=b_{\rm res}+b_{\rm align}$ has two sources & Specification $\times$ residual-coupling design & Coupling activates the within-$X$ residual term, misspecification activates between-$X$ alignment, and the coupled misspecified cell contains both. \\
\addlinespace
Exact independent scores give no first-order PPI gain, but certified samples may be pooled & Exact-independent fifth cell & PPI-optimal power is near zero, whereas direct pooling reduces stochastic error when the synthetic law is certified correct. \\
\addlinespace
Power follows a covariance--variance tradeoff & Paired-score oracle and signed-power curves & In the oracle cell $a=b=c$ and the empirical minimum is near $1/(1+\kappa)$; the signed diagnostic recovers the predicted negative optimum. \\
\addlinespace
Learning, pool estimation, and final calibration are separate resources & Joint $n,N,m$ scan, isolated $N$ scan, and calibration-size sweep & Error follows the predicted $n^{-1}$ rate; the isolated $1/N$ coefficient is positive and close to its leading value, while supervision is flat in $N$; finite-calibration error decreases with $1/m$. \\
\bottomrule
\end{tabular}
\end{table}

\subsection{Direct checks of the deployment derivative and resource scaling}\label{app:theory_followup}
The earlier diagnostics identify null directions and information mechanisms. We add two direct checks of the approximation statements used by the main theory. Both use the controlled score DGP and leave the eight-dataset and HelpSteer2 results unchanged.

\paragraph{Density-weighted deployment derivative.}
For Corollary~\ref{prop:functional_projection}, we evaluate the population conformalization map with 2,048-point Gauss--Legendre quadrature. We perturb the oracle threshold by constant, standardized linear, and standardized quadratic functions at $\epsilon\in\{2^{-1},\ldots,2^{-8}\}$, recomputing the exact population scalar correction at every $\epsilon$. Constant shifts change the deployed threshold by at most $8.88\times10^{-16}$ and the coverage curve by at most $2.22\times10^{-16}$. For the two nonconstant directions, the log--log slopes of the $L_2(P_X)$ first-order remainder over the five smallest perturbations are 1.981 and 2.009. Figure~\ref{fig:post_conformal_derivative} therefore checks both the exact null direction and the second-order remainder, rather than only constructing a null covariance in influence space.

\paragraph{Three-resource scaling.}
We next use the well-specified residual-coupled cell, fix the population oracle power at $0.2541$, and use 30 common seeds per setting. In the joint scan $n\in\{200,400,800,1600\}$ with $N=4n$ and $m=n$, the finite-calibration deployment-error slopes are $-1.159$ for oracle power and $-1.171$ for supervision; the paired-bootstrap interval for the oracle slope is $[-1.364,-0.952]$ and contains the target $-1$. To isolate the pool term, we fix $n=400$ and $m=6400$ and vary $N\in\{400,800,1600,3200,6400\}$. Regressing population-calibrated error on $1/N$ gives coefficient 0.0240 with interval $[0.00643,0.04234]$, compared with the leading coefficient 0.0192; the supervised control, which never reads the pool, is exactly flat. Figure~\ref{fig:three_resource_scaling}, together with the earlier calibration-size experiment, supplies separate numerical evidence for the $n^{-1}$, $N^{-1}$, and $m^{-1}$ terms. The large-$N$ means are noisy relative to the fixed learning floor, so we interpret signs, rates, and scale rather than pointwise equality.

\subsection{Profiled-risk and selection diagnostics}\label{app:profiled_diagnostics}
These experiments evaluate consequences of Proposition~\ref{prop:profiled_risk} and Proposition~\ref{prop:quadratic_discrepancy}; they do not replace fixed unit PPI as the proposed method.

\paragraph{Global-shift audit.}
For each of 240 dataset--seed--candidate comparisons, we add candidate-specific constants before applying the selection criteria. Raw pinball selection changes in 205 comparisons and the complete raw-risk ordering changes in 235. Cross-calibrated pinball, trusted-only and synthetic-assisted coverage moments, and the final deployed thresholds remain invariant up to $8.88\times10^{-16}$. This directly checks the quotient geometry in Proposition~\ref{prop:profiled_risk}; we report it only as an implementation invariant, not as evidence of predictive quality.

\paragraph{Candidate selection.}
We fix six neural candidates---three training powers $\lambda\in\{0,0.5,1\}$ at two checkpoints---and compare raw pinball, cross-calibrated pinball, trusted-only coverage moments, and synthetic-assisted coverage moments. Development data are cross-calibrated, final calibration is untouched, and an Extra-Supervised arm accounts for the additional trusted validation resource. Table~\ref{tab:profiled_selector} and Figure~\ref{fig:profiled_selection} show that the synthetic-assisted selector is worse than fixed unit PPI on Bike, Diamonds, and Transcoding, while its SGEMM interval crosses zero. We therefore retain fixed RCP-PPI and report this study as a negative method result.

\begin{figure}[htbp]
    \centering
    \includegraphics[width=0.88\linewidth]{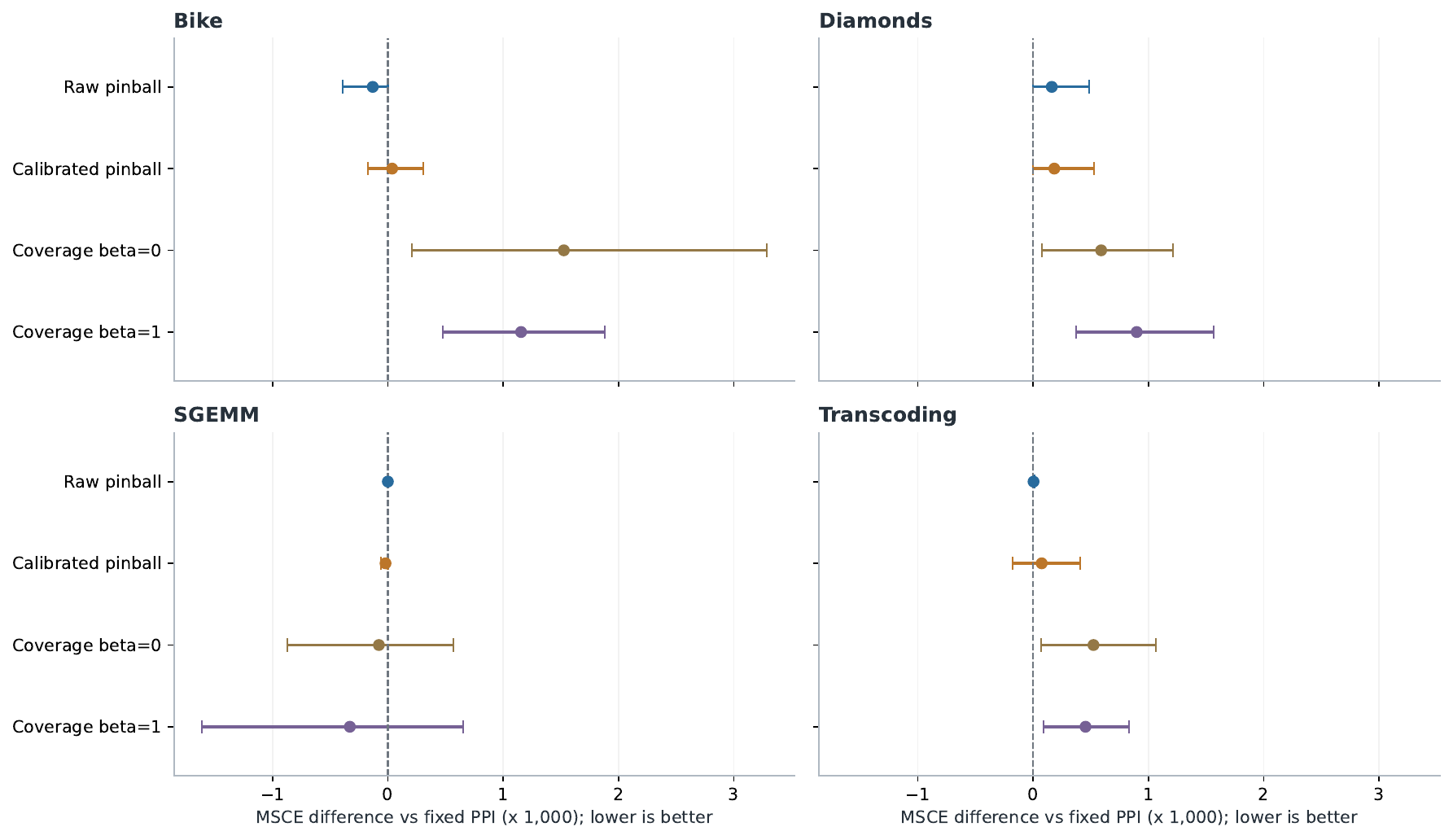}
    \caption{Paired grouped-MSCE changes for the profiled-risk candidate-selection study. Negative values favor the row method. The coverage-based selector does not improve consistently over fixed unit PPI.}
    \label{fig:profiled_selection}
\end{figure}

\begin{table}[htbp]
\centering
\caption{Coverage-moment selection relative to fixed unit PPI. Entries are paired changes in grouped MSCE, multiplied by $10^3$; lower is better. Intervals are descriptive paired-bootstrap 95\% intervals over 15 common splits.}
\label{tab:profiled_selector}
\setlength{\tabcolsep}{4.5pt}
\begin{tabular*}{\linewidth}{@{\extracolsep{\fill}}lrrr@{}}
\toprule
Dataset & Mean $\Delta$ & 95\% interval & Interpretation \\
\midrule
Bike & $+1.155$ & $[+0.476,+1.878]$ & worse \\
Diamonds & $+0.899$ & $[+0.371,+1.569]$ & worse \\
SGEMM & $-0.329$ & $[-1.610,+0.652]$ & inconclusive \\
Transcoding & $+0.455$ & $[+0.092,+0.832]$ & worse \\
\bottomrule
\end{tabular*}
\end{table}

\paragraph{Risk--coverage audit on HelpSteer2.}
Temporary cross-calibration removes global threshold level before measuring pinball risk. Nevertheless, Flash-PPI and Pro-PPI remain worse than supervised RCP by 0.002264 and 0.003030, while their final grouped MSCE improves by 0.001571 and 0.002011. Table~\ref{tab:profiled_helpsteer} therefore exhibits the distinction in Proposition~\ref{prop:profiled_risk}: risk values and gradient norms need not rank a finite candidate set identically. The reconstructed HelpSteer experiment has prompt-cluster dependence, so these are descriptive paired results rather than a new i.i.d. conformal guarantee.

\begin{table}[htbp]
\centering
\caption{HelpSteer2 risk--coverage audit. Entries are PPI minus supervised RCP over 20 paired seeds. Lower is better for pinball risk, MSCE, and log volume; higher is better for minimum-group coverage.}
\label{tab:profiled_helpsteer}
\setlength{\tabcolsep}{3.5pt}
\begin{tabular*}{\linewidth}{@{\extracolsep{\fill}}lrrrr@{}}
\toprule
Method & Cross-cal. pinball & Group MSCE & Min-group cov. & Log volume \\
\midrule
Flash-PPI & $+0.002264$ & $-0.001571$ & $+0.01799$ & $-0.02825$ \\
Pro-PPI & $+0.003030$ & $-0.002011$ & $+0.02410$ & $-0.06575$ \\
\bottomrule
\end{tabular*}
\end{table}

\paragraph{Quadratic-estimator Monte Carlo.}
We test Eq.~\eqref{eq:quadratic_discrepancy_estimator} and Eq.~\eqref{eq:quadratic_discrepancy_variance} in 32 settings with 100,000 repetitions each. The largest absolute bias and variance-formula discrepancies are 2.10 and 2.36 Monte Carlo standard errors. Synthetic correction reduces variance under positive coupling but inflates it under independence or negative coupling, as summarized in Table~\ref{tab:quadratic_variance_ratio} and Figure~\ref{fig:profiled_estimator_variance}. Correct estimation therefore does not imply improved model selection.

\begin{table}[htbp]
\centering
\caption{Variance ratio of the synthetic-assisted quadratic estimator ($\beta=1$) to its trusted-only version ($\beta=0$) across $n\in\{20,50,100,200\}$ and $N=5n$. Values below one indicate variance reduction.}
\label{tab:quadratic_variance_ratio}
\setlength{\tabcolsep}{5pt}
\begin{tabular*}{\linewidth}{@{\extracolsep{\fill}}lcc@{}}
\toprule
True--synthetic coupling & Range of variance ratios & Consequence \\
\midrule
Positive, nonzero moment & $0.45$--$0.57$ & reduction \\
Positive, zero moment & $0.37$--$0.38$ & reduction \\
Independent & $2.29$--$3.18$ & inflation \\
Negative & $3.01$--$4.61$ & inflation \\
\bottomrule
\end{tabular*}
\end{table}

\begin{figure}[htbp]
    \centering
    \includegraphics[width=0.84\linewidth]{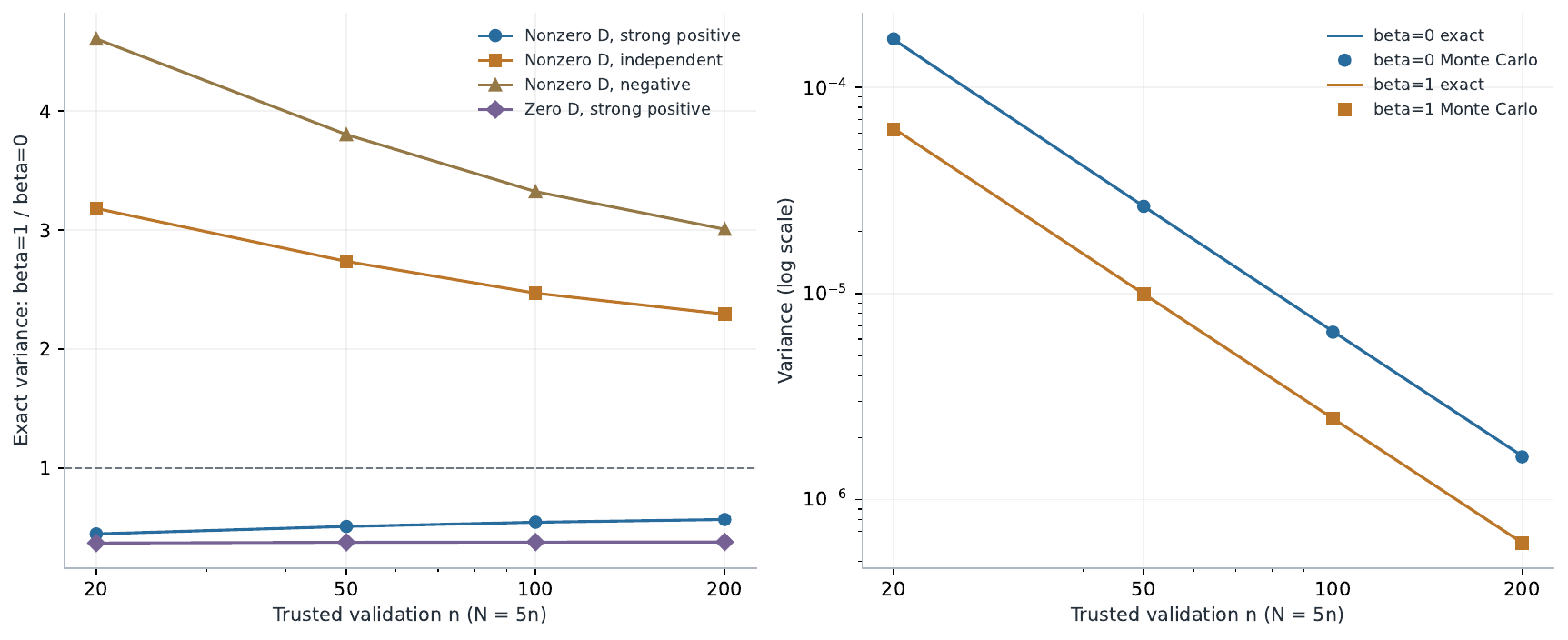}
    \caption{Variance of the unbiased quadratic discrepancy estimator. Synthetic correction helps under positive coupling and hurts when the synthetic moment is independent or negatively coupled.}
    \label{fig:profiled_estimator_variance}
\end{figure}

\section{Additional Theoretical Statements}\label{app:additional_results}
\subsection{Regularity and profiled-risk details}\label{app:profiled_proof_details}
The main text states only the deployment-relevant consequences. The profiled correction satisfies
\begin{equation}\label{eq:profiled_gamma_derivative}
D\gamma_q[h]= -\frac{\E\{f_q(X)h(X)\}}{\E f_q(X)},
\end{equation}
and the Hessian in Eq.~\eqref{eq:profiled_hessian} is positive semidefinite because
\begin{equation}\label{eq:profiled_hessian_quadratic}
\langle h,\mathcal H_qh\rangle
=\inf_{c\in\R}\E[f_q(X)\{h(X)+c\}^2]\ge0.
\end{equation}
In coverage coordinates $a_0=f_\star e$, the oracle map is the codimension-one oblique projection
\begin{equation}\label{eq:projection_operator}
(\Pi a_0)(x)=a_0(x)-\frac{f_\star(x)}{\bar f_\star}\E\{a_0(X)\}.
\end{equation}

\begin{assumption}[Local learning and deployment regularity]\label{ass:deployment_qmd}
\textit{Model and curvature.} The threshold family is fixed dimensional and twice differentiable near an interior stationary point $\theta_0$ of the trusted risk, and $J$ is positive definite. \textit{Learning expansion.} Let $P^\dagger$ denote expectation under the common paired/pool synthetic law. Uniformly over $\lambda$ in a fixed compact set,
\begin{equation}\label{eq:bahadur_main}
\widehat\theta_\lambda-\theta_0
=-J^{-1}\left[(P_n-P)(\xi-\lambda\xi^\dagger)
+\lambda(\widetilde P_N-P^\dagger)\widetilde\xi\right]+r_\lambda,
\end{equation}
where $\E\sup_\lambda\|r_\lambda\|^2=o(n^{-1}+N^{-1})$ and $\E\sup_\lambda\|\widehat\theta_\lambda-\theta_0\|^4=O\{(n^{-1}+N^{-1})^2\}$. \textit{Deployment expansion.} The map from the learned threshold and probability-scale calibration error to $\pi_{\mathcal D,\lambda}$ has the second-order $L_2(P_X)$ expansion induced by Eq.~\eqref{eq:post_conformal_map}. \textit{Sampling.} The three samples are independent, calibration residuals are continuous, $n/N\to\kappa\in[0,\infty)$, and $n/m$ converges to a positive constant.
\end{assumption}

\subsection{Information mechanisms and consequences}
\begin{lemma}[Reparameterization invariance]\label{lem:reparam_invariance}
Under any locally invertible differentiable reparameterization, the random function $\mathsf B\xi$ is unchanged. Consequently, $a$, $b$, $c$, the oracle scalar power, and its leading regret are invariant, although their coordinate matrices change.
\end{lemma}

\begin{corollary}[Conformal null information]\label{cor:conformal_null}
For centered influences,
\begin{equation}\label{eq:projected_covariance_map}
\tr\{G_c\Cov(\xi,\xi^\dagger)\}
=\E\langle\mathsf B\xi,\mathsf B\xi^\dagger\rangle_{L_2(P_X)}.
\end{equation}
Raw covariance can therefore be nonzero while deployed covariance is zero, particularly when the shared information lies only in a global threshold-shift direction.
\end{corollary}

\begin{proposition}[Certified pooling versus bias-robust PPI]\label{prop:pooling_robustness}
Under exact specification and a certified synthetic joint law equal to the target law, direct pooling has leading post-conformal learning error
\begin{equation}\label{eq:certified_pooling_rate}
\frac{a}{n+N}+o(n^{-1}+N^{-1}).
\end{equation}
For a local synthetic-law path satisfying, uniformly for bounded $h$,
\begin{equation}\label{eq:local_synthetic_path}
\sqrt n\,\nabla_\theta\Risk_{Q_n}(q_{\theta_0+h/\sqrt n})=\delta+Jh+o(1),
\end{equation}
with matching derivative and variance limits, pooling incurs post-conformal squared bias
\begin{equation}\label{eq:pooling_local_bias}
\frac{\delta^\top G_c\delta}{n(1+\kappa)^2}+o(n^{-1})
\end{equation}
and leading risk
\begin{equation}\label{eq:pooling_local_risk}
\frac{a}{n+N}+\frac{\Delta^2}{n(1+\kappa)^2},\qquad \Delta^2=\delta^\top G_c\delta.
\end{equation}
It improves on $a/n$ at this local order precisely when $\Delta^2<a(1+\kappa)$. Matched paired--pool PPI remains centered at the trusted target because its two synthetic expectations cancel.
\end{proposition}

\begin{corollary}[Conservative estimable power]\label{cor:plugin_power}
Let $\widehat b-b=o_p(t_n)$, $\widehat c-c=o_p(1)$, and $\widehat\kappa\to_p\kappa$, with $t_n\to0$ and $\sqrt n\,t_n\to\infty$. For $\Lambda=[0,\lambda_{\max}]$,
\begin{equation}\label{eq:sign_gated_power}
\widehat\lambda_{\rm ad}
=\Pi_\Lambda\left\{\frac{(\widehat b-t_n)_+}{(1+\widehat\kappa)\max(\widehat c,\epsilon_n)}\right\},\qquad \epsilon_n\downarrow0.
\end{equation}
If $b\le0$ and $c>0$, the rule is asymptotically inactive. If $b,c>0$ and the oracle is interior, it converges to $\lambda^\star$ and has $o_p(n^{-1})$ leading regret.
\end{corollary}

\begin{corollary}[Oracle trusted-label allocation]\label{cor:label_allocation}
For proportional $N=n/\kappa$, trusted budget $L=n+m$, learning constant $C_{\rm learn}=\min_\lambda\{a-2b\lambda+(1+\kappa)c\lambda^2\}$, and $C_{\rm final}=\tau(1-\tau)C_{\rm cal}$,
\begin{equation}\label{eq:optimal_allocation}
n^\star=L\frac{\sqrt{C_{\rm learn}}}{\sqrt{C_{\rm learn}}+\sqrt{C_{\rm final}}},
\qquad
m^\star=L\frac{\sqrt{C_{\rm final}}}{\sqrt{C_{\rm learn}}+\sqrt{C_{\rm final}}}.
\end{equation}
\end{corollary}

\subsection{Finite-resolution post-conformal evaluation}\label{app:quadratic_estimator}
The profiled-risk identity concerns the full conditional-coverage function. For evaluation with finitely many groups or bounded features, fix $\varphi(X)\in\R^d$ and a candidate threshold $q+\gamma$ independently of the evaluation samples. Let
\begin{equation}\label{eq:feature_moment}
Z=\1\{S\le q(X)+\gamma\}-\tau,
\qquad
\mu_\varphi=\E[Z\varphi(X)],
\qquad
D_\varphi=\|\mu_\varphi\|^2.
\end{equation}
For normalized group indicators, $D_\varphi$ is population grouped MSCE; for general features it measures only the projected component of conditional error.

Let $Z_i^\dagger$ and $\widetilde Z_j$ be paired and pool synthetic boundary indicators at the same threshold. For a fixed evaluation weight $\beta$, define
\[
A_i=\varphi(X_i)(Z_i-\beta Z_i^\dagger),
\qquad
B_j=\varphi(\widetilde X_j)\widetilde Z_j,
\]
and let $\widehat\Sigma_A,\widehat\Sigma_B$ be the unbiased sample covariance matrices.
\begin{proposition}[Unbiased synthetic-assisted quadratic discrepancy]\label{prop:quadratic_discrepancy}
If the paired and pool synthetic laws match and the two evaluation samples are independent, then
\begin{equation}\label{eq:quadratic_discrepancy_estimator}
\widehat D_\beta
=\|\bar A+\beta\bar B\|^2
-\frac{\tr(\widehat\Sigma_A)}{n_v}
-\frac{\beta^2\tr(\widehat\Sigma_B)}{N_v}
\end{equation}
is unbiased for $D_\varphi$. Writing $\Sigma_A=\Var(A_i)$ and $\Sigma_B=\Var(B_j)$, its exact variance is
\begin{align}
\Var(\widehat D_\beta)
={}&4\mu_\varphi^\top\left(\frac{\Sigma_A}{n_v}+\frac{\beta^2\Sigma_B}{N_v}\right)\mu_\varphi
+\frac{2\tr(\Sigma_A^2)}{n_v(n_v-1)}
+\frac{2\beta^4\tr(\Sigma_B^2)}{N_v(N_v-1)}\nonumber\\
&+\frac{4\beta^2\tr(\Sigma_A\Sigma_B)}{n_vN_v}.\label{eq:quadratic_discrepancy_variance}
\end{align}
\end{proposition}
When $\mu_\varphi=0$, the first-order term vanishes and evaluation is intrinsically second order. Synthetic correction can therefore reduce or increase selection noise; the training-stage power rule in Corollary~\ref{cor:optimal_power} does not transfer unchanged to $\beta$.

\section{Technical Results and Proofs}\label{app:proofs}

\subsection{Pinball risk and pre-conformal coverage}\label{app:proof_risk_coverage}
\begin{proof}[Proof of Proposition~\ref{prop:risk_coverage_bridge}]
Fix $x$ and write $r_x(t)=\E[\rho_\tau(S-t)\mid X=x]$. The integral identity for the pinball loss gives
\[
r_x(t)-r_x(q^\star(x))
=\int_{q^\star(x)}^t\{F_x(u)-\tau\}\,du,
\]
where an integral with reversed limits has the usual sign. Suppose first that $t\ge q^\star(x)$ and set $\delta=F_x(t)-\tau\ge0$. Since $F_x$ is nondecreasing and $L_F$-Lipschitz,
\[
F_x(u)-\tau\ge\{\delta-L_F(t-u)\}_+
\]
for $u\in[q^\star(x),t]$. Integrating over the support of the right-hand side gives
\[
r_x(t)-r_x(q^\star(x))\ge\frac{\delta^2}{2L_F}.
\]
The case $t<q^\star(x)$ is identical after setting $\delta=\tau-F_x(t)$ and reversing the integral. Hence
\[
\{F_x(t)-\tau\}^2\le2L_F\{r_x(t)-r_x(q^\star(x))\}.
\]
Substituting $t=q(x)$ and integrating over $P_X$ proves the claim.
\end{proof}

\subsection{Exact MSCE identities and probability-scale calibration}\label{app:finite_calibration}
\begin{proof}[Proof of Proposition~\ref{prop:msce_exact}]
For a fixed deployment, add and subtract $U_{\mathcal D}=\E_X\pi_{\mathcal D}(X)$ and use the orthogonality of a centered random variable and a constant:
\[
\E_X[(\pi_{\mathcal D}(X)-\tau)^2]
=\Var_X\{\pi_{\mathcal D}(X)\}+(U_{\mathcal D}-\tau)^2.
\]
Conditional on $\widehat q$, let $R=S-\widehat q(X)$ have continuous CDF $F_R$. Since
$U_{\mathcal D}=F_R(\widehat\gamma)$ and the calibration residuals are i.i.d. from $F_R$, the probability integral transform gives
$F_R(\widehat\gamma)\sim\operatorname{Beta}(k,m+1-k)$. Its mean and variance yield Eq.~\eqref{eq:beta_setup}.
\end{proof}

\begin{proof}[Proof of Eq.~\eqref{eq:realized_ensemble_decomp}]
For each fixed $x$, apply the bias--variance identity to $\pi_{\mathcal D}(x)$ around $\tau$; integrating the resulting equality over $P_X$ proves the claim.
\end{proof}

\begin{proof}[Proof of Proposition~\ref{prop:shift_invariance}]
Replacing $q$ by $q+c$ shifts every calibration residual by $-c$, hence shifts its $k$th order statistic by $-c$. The sum $q(x)+\widehat\gamma(q)$ and therefore the prediction set are unchanged.
\end{proof}

\begin{proof}[Proof of Proposition~\ref{prop:adaptive_validity}]
Conditional on $(\mathcal D_L,\mathcal D_U^X)$ and the fixed fitted objects and the auxiliary learning randomness, $\widehat\lambda$ and $\widehat q$ are fixed, while the $m$ calibration residuals and the test residual are exchangeable. Under continuity the rank is uniform. With ties, the conservative $k$th order statistic still gives the displayed lower bound; randomized ranks can recover exact rank validity but are not used for the Beta identity in Proposition~\ref{prop:msce_exact}. The infinite-set convention covers $k=m+1$.
\end{proof}

\begin{proposition}[Exact transport and a finite-sample bridge]\label{prop:finite_cal_transport}
Fix a measurable $q$ independent of calibration and let $R=S-q(X)$ have a continuous strictly increasing CDF $F_R$. Define
\[
U_{(k)}=F_R(\widehat\gamma),\quad
g_x(u)=F_{R\mid X=x}\{F_R^{-1}(u)\},\quad
u_0=F_R(0).
\]
Then $U_{(k)}\sim\operatorname{Beta}(k,m+1-k)$ and
\[
\pi_{\mathcal D}(x)=g_x(U_{(k)}),
\qquad F_{S\mid X}(q(x)\mid x)=g_x(u_0).
\]
If $\sup_r f_{R\mid X=x}(r)/f_R(r)\le\Lambda$, then, with
$v_m=\Var(U_{(k)})$ and $\tau_m=\E U_{(k)}$,
\begin{align}
\left|\E_{\mathcal D_C}\pi_{\mathcal D}(x)-F_{S\mid X}(q(x)\mid x)\right|
&\le\Lambda\{|\tau_m-u_0|+\sqrt{v_m}\},\label{eq:finite_bias_bridge}\\
\E_{\mathcal D_C}[(\pi_{\mathcal D}(x)-\tau)^2]
&\le2\{F_{S\mid X}(q(x)\mid x)-\tau\}^2
+2\Lambda^2\{(\tau_m-u_0)^2+v_m\}.\label{eq:finite_msce_bridge}
\end{align}
\end{proposition}

\begin{proof}
The probability integral transform gives the Beta law and the displayed identities. The density-ratio bound implies
$g_x'(u)=f_{R\mid X=x}(F_R^{-1}(u))/f_R(F_R^{-1}(u))\le\Lambda$, hence $g_x$ is $\Lambda$-Lipschitz. Jensen's inequality and the Beta moments prove Eq.~\eqref{eq:finite_bias_bridge}; applying $(a+b)^2\le2a^2+2b^2$ proves Eq.~\eqref{eq:finite_msce_bridge}.
\end{proof}

\subsection{Profiled risk and post-conformal geometry}
\begin{proof}[Proof of Proposition~\ref{prop:profiled_risk}]
Shift invariance follows by changing variables in the scalar infimum:
$\inf_\gamma\Risk(q+c+\gamma)=\inf_{\gamma'}\Risk(q+\gamma')$.
For fixed $(q,\gamma)$, the directional derivative of the pinball risk is
\[
D\Risk(q+\gamma)[h]
=\E[\{F_X(q(X)+\gamma)-\tau\}h(X)].
\]
At the locally unique minimizer $\gamma_q$, the scalar first-order condition is
$\E\{\pi_q(X)-\tau\}=0$. The envelope theorem therefore removes the derivative of $\gamma_q$ and gives Eq.~\eqref{eq:profiled_gradient}. The $L_2(P_X)$ gradient and Eq.~\eqref{eq:profiled_msce_identity} follow immediately.

For the second derivative, differentiate the calibration equation
$\E F_X\{q(X)+\gamma_q\}=\tau$ along $q+th$ at $t=0$. This gives
\[
0=\E[f_q(X)\{h(X)+D\gamma_q[h]\}],
\]
which proves Eq.~\eqref{eq:profiled_gamma_derivative}. Differentiating the gradient $\pi_q-\tau$ then yields Eq.~\eqref{eq:profiled_hessian}. Finally,
\[
\langle h,\mathcal H_qh\rangle
=\E(f_qh^2)-\frac{\{\E(f_qh)\}^2}{\E f_q}
=\inf_{c\in\R}\E[f_q(X)\{h(X)+c\}^2],
\]
where the minimizing constant is $c=-\E(f_qh)/\E f_q$. This proves positive semidefiniteness.
\end{proof}

\begin{proof}[Proof of Corollary~\ref{prop:functional_projection}]
Write $\gamma_{e,v}=\Gamma(q^\star+e,\tau+v)$ and
$z_{e,v}(x)=e(x)+\gamma_{e,v}$. A Taylor expansion of the marginal constraint, uniform in the stated neighborhood, gives
\[
v=\E[f_\star(X)z_{e,v}(X)]
+O\{(\|e\|_\infty+|\gamma_{e,v}|)^2\}.
\]
Local uniqueness and $\bar f_\star>0$ imply
$|\gamma_{e,v}|\lesssim\|e\|_\infty+|v|$. Solving the preceding equation yields
\[
\gamma_{e,v}
=-\frac{\E(f_\star e)}{\bar f_\star}
+\frac{v}{\bar f_\star}
+O\{(\|e\|_\infty+|v|)^2\}.
\]
A pointwise Taylor expansion of $F_x$ at $q^\star(x)$, together with the uniform bound on the density derivative, gives an $L_2(P_X)$ remainder of order
$(\|e\|_\infty+|v|)^2$. Substitution proves Eqs.~\eqref{eq:functional_coverage_expansion}--\eqref{eq:functional_operator}. Equation~\eqref{eq:projection_operator} follows by setting $a_0=f_\star e$.
\end{proof}

For a parametric path $e_v(x)=d(x)^\top v$, Corollary~\ref{prop:functional_projection} gives
$D\mathcal P_{(q_0,\tau)}[e_v,0]=f_c d_c^\top v$. This proves the finite-dimensional derivative asserted around Eq.~\eqref{eq:projected_geometry} and shows that $G_c$ is positive semidefinite. It can be singular because constant-shift directions are removed; all correlation statements therefore use the induced semi-inner product, or equivalently the quotient by $\ker(G_c)$.

\subsection{Stationary-target PPI asymptotics}\label{app:parametric_proofs}

\begin{assumption}[Stationary local model and uniform asymptotic linearity]\label{ass:stationary_regular}
Fix a compact admissible power set $\Lambda$. Conditional on the fixed fitted objects, the three splits and test pair are mutually independent, $(X,S^\dagger)\stackrel d=(\widetilde X,\widetilde S)$, $n/N\to\kappa\in[0,\infty)$, and $n/m\to\nu\in(0,\infty)$. Calibration residuals have a continuous marginal CDF. Ties may be handled conservatively for Proposition~\ref{prop:adaptive_validity}, but the Beta expansion below is asserted only under continuity.

The parameter space contains a neighborhood of an interior stationary point $\theta_0$; $J$ is positive definite; $q_\theta$ is twice continuously differentiable with integrable fourth-moment envelopes; and the true and synthetic boundary densities are locally regular. Uniformly over $\lambda\in\Lambda$,
\begin{equation}\label{eq:bahadur_representation}
\widehat\theta_\lambda-\theta_0
=-J^{-1}Z_\lambda+R_\lambda,
\qquad
Z_\lambda=(P_n-P)(\xi-\lambda\xi^\dagger)
+\lambda(\widetilde P_N-P^\dagger)\widetilde\xi,
\end{equation}
with
\begin{equation}\label{eq:bahadur_remainders}
\frac{\E[\sup_{\lambda\in\Lambda}\|R_\lambda\|^2]}{n^{-1}+N^{-1}}\to0,
\qquad
\E\left[\sup_{\lambda\in\Lambda}\|\widehat\theta_\lambda-\theta_0\|^4\right]
=O\{(n^{-1}+N^{-1})^2\}.
\end{equation}
\end{assumption}

\begin{remark}[Proof level of the learning expansion]\label{rem:ual_scope}
Assumption~\ref{ass:stationary_regular} is a high-level uniform asymptotic-linearity condition. The contribution of Theorem~\ref{thm:realized_msce_expansion} is the downstream post-conformal geometry and resource decomposition conditional on such a learning expansion, not a claim that arbitrary signed ERM automatically satisfies it.
\end{remark}

For the pinball risk, write
\[
g_\theta(x,s)=\{\1(s\le q_\theta(x))-\tau\}\nabla_\theta q_\theta(x).
\]

\begin{proposition}[A sufficient fixed-dimensional one-step construction]\label{prop:onestep_sufficient}
Suppose $q_\theta(x)=d(x)^\top\theta$ has fixed dimension and bounded fourth-moment feature envelope; true and synthetic conditional CDFs are continuously differentiable in a uniform local tube with Lipschitz densities; $J$ is positive definite; $\widetilde\theta$ is a root-$n$ trusted-only pilot; and $\widehat J\to_pJ$. Define
\begin{equation}\label{eq:onestep_estimator}
\widehat\theta_\lambda^{\rm os}
=\widetilde\theta-\widehat J^{-1}\left[
P_ng_{\widetilde\theta}(X,S)
+\lambda\{\widetilde P_Ng_{\widetilde\theta}(\widetilde X,\widetilde S)-P_ng_{\widetilde\theta}(X,S^\dagger)\}
\right].
\end{equation}
Under the standard fixed-dimensional quantile-score stochastic-equicontinuity conditions, Eq.~\eqref{eq:bahadur_representation} holds uniformly on compact $\Lambda$ with $R_\lambda=o_p(n^{-1/2}+N^{-1/2})$. Bounded higher moments give the expected remainder controls in Eq.~\eqref{eq:bahadur_remainders}.
\end{proposition}

\begin{proof}
Expand the one-step score at $\theta_0$. The trusted population score vanishes by stationarity, while equality in law of paired and pool synthetic observations cancels their population means. Uniform stochastic equicontinuity controls replacement of $\widetilde\theta$ by $\theta_0$; consistency of $\widehat J$ and the root-$n$ pilot make the remaining terms second order. Compactness of $\Lambda$ makes the affine dependence on $\lambda$ uniform. Standard fourth-moment maximal inequalities yield the stated expected controls.
\end{proof}

\begin{proposition}[Stability boundary of signed pinball ERM]\label{prop:signed_objective_stability}
For fixed scores $s,s^\dagger$, let $g_\lambda(t)=\rho_\tau(s-t)-\lambda\rho_\tau(s^\dagger-t)$. If $\lambda>1$, then $g_\lambda(t)\to-\infty$ as $t\to\pm\infty$; if $\lambda=1$, both tails converge to finite constants and $g_1$ is not coercive. Thus the local theorem concerns a stable root or one-step update; unrestricted rich-class signed ERM may be ill posed.
\end{proposition}

\begin{proof}
For $t$ above both scores, $g_\lambda(t)=(1-\tau)\{(1-\lambda)t-s+\lambda s^\dagger\}$; below both scores, $g_\lambda(t)=\tau\{s-\lambda s^\dagger-(1-\lambda)t\}$. The conclusions follow directly.
\end{proof}

\begin{lemma}[Second-moment consequence]\label{lem:ppi_bahadur}
Under Assumption~\ref{ass:stationary_regular}, uniformly over $\lambda\in\Lambda$,
\begin{align}\label{eq:theta_second_moment}
\E[(\widehat\theta_\lambda-\theta_0)(\widehat\theta_\lambda-\theta_0)^\top]
={}&\frac1nJ^{-1}\Var(\xi-\lambda\xi^\dagger)J^{-1}\\
&+\frac{\lambda^2}{N}J^{-1}\Var(\widetilde\xi)J^{-1}
+o(n^{-1}+N^{-1}).\nonumber
\end{align}
\end{lemma}

\begin{proof}
The two empirical terms are independent and centered. Expand Eq.~\eqref{eq:bahadur_representation}; the $L_2$ remainder bound and Cauchy--Schwarz make every remainder and cross term negligible uniformly over $\Lambda$.
\end{proof}

\begin{proof}[Proof of Theorem~\ref{thm:realized_msce_expansion}]
Conditional on $\widehat\theta_\lambda$, the probability-scale calibration variable $U=F_{\widehat\theta_\lambda}(\widehat\gamma)$ is $\operatorname{Beta}(k,m+1-k)$ and independent of the learning samples. Assumption~\ref{ass:deployment_qmd} and Corollary~\ref{prop:functional_projection} give
\[
\pi_{\mathcal D,\lambda}(X)-\pi_0(X)
=f_c(X)d_c(X)^\top(\widehat\theta_\lambda-\theta_0)
+\frac{f_c(X)}{\E f_c(X)}(U-\tau)+r_X.
\]
The fourth-moment controls localize the expansion with complement probability $o(n^{-1}+N^{-1}+m^{-1})$; bounded coverage and uniform integrability handle the complement. Taking the expected squared $L_2(P_X)$ norm, applying Lemma~\ref{lem:ppi_bahadur}, and using independence yields Eq.~\eqref{eq:stationary_stoch_expansion}. Centering over the procedure replaces $\E(U-\tau)^2$ by $\Var(U)$ and gives Eq.~\eqref{eq:stationary_variance_expansion}. Corollary~\ref{cor:exact_specification} follows from $\pi_0=\tau$ and the exact MSCE decomposition in Proposition~\ref{prop:msce_exact}.
\end{proof}

\begin{proof}[Proof of Lemma~\ref{lem:reparam_invariance} and Corollary~\ref{cor:conformal_null}]
For a local linear change of coordinates $\theta=A\eta$, one has $d_\eta=A^\top d_\theta$, $J_\eta=A^\top J_\theta A$, and $\xi_\eta=A^\top\xi_\theta$. Substitution in Eq.~\eqref{eq:post_conformal_map} gives $\mathsf B_\eta\xi_\eta=\mathsf B_\theta\xi_\theta$. A differentiable reparameterization reduces to its invertible Jacobian locally. The Hessian transforms by congruence because the risk gradient vanishes at the stationary target, so the second-coordinate-derivative term from a nonlinear change of variables drops out. Equation~\eqref{eq:projected_covariance_map} follows by polarization and $G_c=(\mathsf B)^*\mathsf B$. A global-shift score $J w$ is in the kernel because $d_c(X)^\top w=0$.
\end{proof}

\subsection{Information boundary, power, and allocation}
\begin{proof}[Proof of Proposition~\ref{prop:mechanism_decomposition} and Corollary~\ref{cor:no_gain}]
The law of total covariance decomposes $\Cov(\xi,\xi^\dagger)$ into $\E\{\Cov(\xi,\xi^\dagger\mid X)\}$ and $\Cov\{\mu(X),\mu^\dagger(X)\}$. Applying $M\mapsto\tr(G_cM)$ proves Eq.~\eqref{eq:two_mechanisms}. Under exact specification, $\mu(X)=\{F_X(q^\star(X))-\tau\}d(X)=0$. Conditional independence also makes the conditional covariance zero, so $b=0$ and Eq.~\eqref{eq:power_quadratic} is minimized at zero over nonnegative powers.
\end{proof}

\begin{proof}[Proof of Proposition~\ref{prop:pooling_robustness}]
Under certified correctness, trusted and synthetic estimating scores are independent and identically distributed. The pooled score average has covariance $\Var(\xi)/(n+N)$, and the post-conformal map gives Eq.~\eqref{eq:certified_pooling_rate}. In the robust model, conditional independence invokes the zero-power conclusion of Corollary~\ref{cor:no_gain}.

For the local path, Eq.~\eqref{eq:local_synthetic_path} supplies both the first-order score displacement and the common limiting derivative $J$; the variance assumption supplies the pooled covariance limit. The synthetic fraction is $N/(n+N)=1/(1+\kappa)+o(1)$, so the pooled stationary point is shifted by
$-J^{-1}\delta/\{(1+\kappa)\sqrt n\}+o(n^{-1/2})$. Applying the post-conformal map gives Eq.~\eqref{eq:pooling_local_bias}. Adding the pooled variance $a/(n+N)$ and comparing with $a/n$ gives the threshold. If paired and pool synthetic observations follow the same $Q_n$, their population score contributions cancel in Eq.~\eqref{eq:ppi_obj_main}.
\end{proof}

\begin{proof}[Proof of Corollary~\ref{cor:optimal_power}]
Expanding the learning variance in Theorem~\ref{thm:realized_msce_expansion} gives Eq.~\eqref{eq:power_quadratic}. Differentiation and clipping yield Eq.~\eqref{eq:optimal_lambda}. Cauchy--Schwarz in the semi-inner product induced by $G_c$ gives $|b|\le\sqrt{ac}$; substitution gives the minimum, and completing the square gives Eq.~\eqref{eq:lambda_regret}.
\end{proof}

\begin{corollary}[Paired-score oracle check]\label{cor:paired_oracle_check}
If $\xi^\dagger=\xi$ almost surely and the pool has the same boundary-score law, then $a=b=c$ and
\begin{equation}\label{eq:paired_oracle_check}
\lambda^\star=\frac1{1+\kappa},
\qquad
L(\lambda^\star)=\frac{a}{n+N}.
\end{equation}
\end{corollary}

\subsection{Cross-fitted estimation of projected boundary information}\label{app:power_estimation}
Partition the paired sample into a fixed number of folds. On each complement, fit a trusted-only pilot, nearby quantile heads for conditional boundary density, the derivative map, and a positive-semidefinite estimate of $G_c$. Evaluate both paired boundary scores and the entire synthetic pool using that fold-specific pilot. Center within fold and average the resulting bilinear forms to obtain $\widehat b$; use fold-specific pool scores for $\widehat c$.

\begin{proposition}[Consistency of cross-fitted projected constants]\label{prop:cf_bc_consistency}
If the fold-specific threshold, derivative, density, and geometry estimates converge in the corresponding $L_2/L_4$ and operator norms, the number of folds is fixed, and relevant fourth moments are uniformly bounded, then $\widehat a\to_p a$, $\widehat b\to_p b$, and $\widehat c\to_p c$.
\end{proposition}

\begin{proof}
Conditional on each off-fold pilot, validation observations and the independent pool obey their target laws. Conditional laws of large numbers, fold-mean consistency, indicator convergence away from vanishing boundary neighborhoods, and uniform integrability yield the required population bilinear forms and consistency.
\end{proof}

\begin{proof}[Proof of Corollary~\ref{cor:plugin_power}]
If $b<0$, consistency places $\widehat b-t_n$ below zero with probability tending to one. If $b=0$, the stronger $o_p(t_n)$ condition gives the same conclusion. If $b,c>0$, consistency and $t_n,\epsilon_n\to0$ give convergence of the unclipped ratio to $b/\{(1+\kappa)c\}$; interiority removes clipping. Equation~\eqref{eq:lambda_regret} then gives $o_p(n^{-1})$ regret.
\end{proof}

\begin{proposition}[Boundary disagreement is a coarse upper bound]\label{prop:boundary_diagnostic}
At exact specification, let $D_\tau=\1\{\1(S\le q^\star(X))\ne\1(S^\dagger\le q^\star(X))\}$. Then
\begin{equation}\label{eq:boundary_diagnostic_bound}
\tr\{G_c\Var(\xi-\xi^\dagger)\}
\le\E[D_\tau d(X)^\top G_cd(X)].
\end{equation}
This diagnostic does not identify the sign of $b$.
\end{proposition}

\begin{proof}
The score difference is the boundary-indicator difference times $d(X)$. For positive-semidefinite $G_c$, $\tr\{G_c\Var(V)\}\le\E(V^\top G_cV)$.
\end{proof}

\begin{proof}[Proof of Corollary~\ref{cor:label_allocation}]
Minimize $C_{\rm learn}/n+C_{\rm final}/(L-n)$. The first-order condition and strict convexity yield Eq.~\eqref{eq:optimal_allocation}.
\end{proof}

\begin{proposition}[$X$-only information ceiling]\label{prop:xonly_ceiling}
In a pseudo-true regime with $S^\dagger\perp S\mid X$, let $\mu(X)=\E(\xi\mid X)$ and $a_X=\tr\{G_c\Var(\mu(X))\}$. Then
\begin{equation}\label{eq:xonly_ceiling}
a=a_X+\tr[G_c\E\{\Var(\xi\mid X)\}],
\qquad
\frac{b^2}{c}\le a_X,
\end{equation}
with the ratio zero when $c=0$. Thus a scalar $X$-only control can remove at most the between-$X$ part of the leading learning constant; a finite pool further scales the attainable reduction by $1/(1+\kappa)$.
\end{proposition}

\begin{proof}
Conditional independence gives $\Cov(\xi,\xi^\dagger)=\Cov\{\mu(X),\E(\xi^\dagger\mid X)\}$. Apply the law of total variance and Cauchy--Schwarz after mapping both variables by $(G_c)^{1/2}$.
\end{proof}

\subsection{Synthetic-law mismatch}\label{app:synthetic_mismatch}
\begin{proposition}[Bias from paired--pool mismatch]\label{prop:synthetic_mismatch}
Let $A_c=\E[f_c(X)^2d_c(X)d_c(X)^\top]$, so $G_c=J^{-1}A_cJ^{-1}$. Define
\[
M_t(\theta)=\E_{{\rm pool},t}g_\theta(\widetilde X,\widetilde S)
-\E_{{\rm paired},t}g_\theta(X,S^\dagger),
\qquad \mu_{\dagger,t}=M_t(\theta_0)\to0.
\]
Assume $M_t$ is differentiable at $\theta_0$, its derivative is locally equicontinuous, and
$K_{\lambda,t}=J+\lambda\dot M_t(\theta_0)$ remains nonsingular. In the well-specified local model,
\begin{equation}\label{eq:synthetic_mismatch_parameter}
\theta_{\lambda,t}-\theta_0
=-\lambda K_{\lambda,t}^{-1}\mu_{\dagger,t}
+o(\|\mu_{\dagger,t}\|),
\end{equation}
and the post-conformal squared bias is
\begin{equation}\label{eq:synthetic_mismatch_bias}
\lambda^2\mu_{\dagger,t}^\top
K_{\lambda,t}^{-\top}A_cK_{\lambda,t}^{-1}
\mu_{\dagger,t}
+o(\|\mu_{\dagger,t}\|^2).
\end{equation}
If additionally $\|\dot M_t(\theta_0)\|\to0$, this reduces to
$\lambda^2\mu_{\dagger,t}^\top G_c\mu_{\dagger,t}+o(\|\mu_{\dagger,t}\|^2)$.
\end{proposition}

\begin{proof}
The expected PPI gradient is $\nabla\Risk_P(\theta)+\lambda M_t(\theta)$. A local implicit expansion around $\theta_0$ has derivative $K_{\lambda,t}$ and value $\lambda\mu_{\dagger,t}$, yielding Eq.~\eqref{eq:synthetic_mismatch_parameter}. The parameter perturbation has deployed squared norm $v^\top A_cv$; substitution gives Eq.~\eqref{eq:synthetic_mismatch_bias}. The final reduction follows from $K_{\lambda,t}\to J$.
\end{proof}

For a fixed nonzero paired--pool mismatch, the population objective generally targets a shifted parameter and increasing $n$ or $N$ does not restore the trusted target. Proposition~\ref{prop:synthetic_mismatch} does not claim that its local remainder formula describes arbitrarily large mismatch; importance weighting or an explicit transport model is then required.

\subsection{Quadratic post-conformal discrepancy}
\begin{proof}[Proof of Proposition~\ref{prop:quadratic_discrepancy}]
Matched paired and pool synthetic laws imply
\[
\E(\bar A+\beta\bar B)
=\E[\varphi(X)\{Z-\beta Z^\dagger\}]
+\beta\E[\varphi(X)Z^\dagger]
=\mu_\varphi.
\]
For any i.i.d. vector sample $V_1,\ldots,V_r$ with mean $\mu_V$, the identity
\[
\left\|\bar V\right\|^2-\frac{\tr(\widehat\Sigma_V)}{r}
=\frac{1}{r(r-1)}\sum_{i\ne j}V_i^\top V_j
\]
shows that the corrected squared mean is unbiased for $\|\mu_V\|^2$. Applying this identity to the $A$ and $B$ samples and retaining the independent cross term $2\beta\bar A^\top\bar B$ proves Eq.~\eqref{eq:quadratic_discrepancy_estimator}.

For the variance, use the Hoeffding decomposition of each second-order inner-product U-statistic. Its linear projection has variance $4\mu_\varphi^\top\Sigma_A\mu_\varphi/n_v$ for the $A$ sample and $4\beta^2\mu_\varphi^\top\Sigma_B\mu_\varphi/N_v$ for the $B$ sample. The two degenerate components contribute
$2\tr(\Sigma_A^2)/\{n_v(n_v-1)\}$ and
$2\beta^4\tr(\Sigma_B^2)/\{N_v(N_v-1)\}$. Independence of the two samples makes the remaining cross component orthogonal to these terms, with variance
$4\beta^2\tr(\Sigma_A\Sigma_B)/(n_vN_v)$. Adding these mutually orthogonal contributions proves the exact variance formula in Eq.~\eqref{eq:quadratic_discrepancy_variance}.
\end{proof}

\section{Supplementary Learning-Theoretic Guarantees}\label{app:learning_theory}
The main paper gives a sharp local fixed-dimensional analysis. This section records two additional consistency results not used in Theorem~\ref{thm:realized_msce_expansion}: a global fixed-power bound and a localized nonparametric extension. They establish complementary learnability guarantees rather than the post-conformal efficiency comparisons developed above.

\subsection{A distribution-free global excess-risk bound}
Let $\rho_\tau$ be $L_\rho\le1$ Lipschitz in its prediction argument. Assume $|q(X)|\le B_q$, $|S|\le B_S$, and $|S^\dagger|,|\widetilde S|\le B_\dagger$. Set $B_0=B_q+B_S$, $B_1=B_q+B_\dagger$, and let $q_{\mathcal Q}^\star\in\argmin_{q\in\mathcal Q}\Risk(q)$.

\begin{theorem}[Uniform PPI-ERM bound]\label{thm:global_ppi_bound}
For fixed $\lambda$ and $\delta\in(0,1)$, with probability at least $1-\delta$,
\begin{align}\label{eq:global_ppi_bound}
\Risk(\widehat q_\lambda)-\Risk(q_{\mathcal Q}^\star)
\le{}&4(1+|\lambda|)L_\rho\Rad_n(\mathcal Q)
+4|\lambda|L_\rho\Rad_N(\mathcal Q)\\
&+2(B_0+|\lambda|B_1)\sqrt{\frac{\log(4/\delta)}{2n}}
+2|\lambda|B_1\sqrt{\frac{\log(4/\delta)}{2N}}.
\nonumber
\end{align}
\end{theorem}

\begin{proof}
The deviation from the true risk is the sum of a paired empirical process and an independent pool process. The ERM reduction, symmetrization, contraction, bounded-difference concentration, and a union bound give Eq.~\eqref{eq:global_ppi_bound}.
\end{proof}
This bound establishes consistency rather than dominance; the sharp local improvement criterion is given by Corollaries~\ref{cor:no_gain} and~\ref{cor:optimal_power}.

\subsection{A radius-dependent nonparametric extension}\label{app:nonparametric}
Let $q^\star\in\mathcal Q$ and $\mathcal E(q)=\Risk(q)-\Risk(q^\star)$. For $\lambda=1$, define
\[
\Delta_q=(\ell_q-\ell_{q^\star})(X,S)-(\ell_q-\ell_{q^\star})(X,S^\dagger),
\qquad
\widetilde g_q=(\ell_q-\ell_{q^\star})(\widetilde X,\widetilde S).
\]
Knight's identity gives
\[
\Delta_q=\int_{q^\star(X)}^{q(X)}\{\1(S\le u)-\1(S^\dagger\le u)\}\,du.
\]
Let $D_q$ indicate disagreement somewhere on this segment, so
$|\Delta_q|\le|q-q^\star|D_q$. Assume $|q-q^\star|\le B_Q$ and local curvature
$\mathcal E(q)\ge c_0\|q-q^\star\|_2^2$. Define
\begin{equation}\label{eq:eta_radius}
\eta(r)=\sup_{\substack{q:\mathcal E(q)\le r\\q\ne q^\star}}
\frac{\E[(q-q^\star)^2D_q]}{\E[(q-q^\star)^2]}.
\end{equation}

\begin{assumption}[Localized VC-type deviations]\label{ass:localized_vc}
There are constants $A\ge e,D\ge1,C_0$ such that, simultaneously for relevant $r$ and with probability at least $1-\delta$,
\begin{align}\label{eq:localized_deviations}
\sup_{\mathcal E(q)\le r}|(P_n-P)\Delta_q|
&\le C_0\left[\sqrt{\frac{\eta(r)rV_n}{c_0n}}+\frac{B_QV_n}{n}\right],\\
\sup_{\mathcal E(q)\le r}|(\widetilde P_N-P^\dagger)\widetilde g_q|
&\le C_0\left[\sqrt{\frac{rV_N}{c_0N}}+\frac{B_QV_N}{N}\right],
\nonumber
\end{align}
where $V_n=D\log(An)+\log(4/\delta)$ and $V_N=D\log(AN)+\log(4/\delta)$.
\end{assumption}

\begin{theorem}[Radius-dependent fixed point]\label{thm:nonparam_fixed_point}
If $r_\star$ satisfies, for every $r\ge r_\star$,
\begin{equation}\label{eq:nonparam_fixed_condition}
C_0\left[\sqrt{\frac{\eta(r)rV_n}{c_0n}}+\sqrt{\frac{rV_N}{c_0N}}
+\frac{B_QV_n}{n}+\frac{B_QV_N}{N}\right]\le r/4,
\end{equation}
then with probability at least $1-\delta$,
$\mathcal E(\widehat q_1)\le r_\star$.
If $\eta(r)\le\bar\eta$ locally, one may take
\begin{equation}\label{eq:nonparam_rate_new}
r_\star\lesssim
\frac{\bar\eta V_n}{c_0n}+\frac{V_N}{c_0N}
+\frac{B_QV_n}{n}+\frac{B_QV_N}{N}.
\end{equation}
\end{theorem}

\begin{proof}
The ERM basic inequality bounds $\mathcal E(\widehat q_1)$ by the two empirical deviations. Applying Assumption~\ref{ass:localized_vc} on dyadic excess-risk shells excludes every shell above $r_\star$ via Eq.~\eqref{eq:nonparam_fixed_condition}; solving the resulting inequalities gives Eq.~\eqref{eq:nonparam_rate_new}.
\end{proof}

\begin{corollary}[A local bridge from nonparametric risk to coverage]\label{cor:nonparam_postconformal}
Under the smoothness conditions of Corollary~\ref{prop:functional_projection} and local curvature, suppose additionally that
$\|\widehat q_1-q^\star\|_\infty\le\varepsilon_0$ on the event of interest, where $\varepsilon_0$ is inside the proposition's local neighborhood. Then
\begin{equation}\label{eq:nonparam_postconformal}
\|\mathcal P(\widehat q_1,\tau)-\tau\|_2^2
\le C\,\mathcal E(\widehat q_1).
\end{equation}
Consequently, if the required sup-norm localization also holds with high probability, the population-conformalized conditional error is $O_{\Pr}(r_\star)$. Independent finite calibration is governed separately by Proposition~\ref{prop:msce_exact} and Proposition~\ref{prop:finite_cal_transport}; it is not absorbed into the learning-risk rate. The result follows by applying Corollary~\ref{prop:functional_projection} to $e=q-q^\star$, using $\|e\|_4^2\le\varepsilon_0\|e\|_2$, and then invoking local curvature.
\end{corollary}

\end{document}